\documentclass[letterpaper,10pt]{article}
\usepackage{amssymb,amsmath,amsthm,turnstile,txfonts}
\usepackage{mathrsfs}
\usepackage{color}
\usepackage[dvipsnames]{xcolor}
\usepackage{subfig}
\usepackage{graphicx}
\usepackage{url}
\usepackage{multirow}
\usepackage[left=1in,top=1in,right=1in,bottom=1in,letterpaper]{geometry}
\usepackage{setspace}
\usepackage{enumerate}
\usepackage{algorithm}
\usepackage{algpseudocode}

\usepackage{hyperref}
\hypersetup{
  colorlinks   = true,    
  urlcolor     = blue,    
  linkcolor    = blue,    
  citecolor    = red,     
}

\allowdisplaybreaks

\newtheorem{theorem}{{\bf Theorem}}
\newtheorem{corollary}{{\bf Corollary}}

\newtheorem{remark}{\bf Remark}

\newtheorem{lemma}{\bf Lemma}

\newcommand{\cL}{\mathcal{L}}
\DeclareMathOperator{\supp}{supp}

\definecolor{DukeBlue}{HTML}{001A57}
\definecolor{DarkRed}{rgb}{0.75, 0.0, 0.0}
\definecolor{DarkGreen}{rgb}{0.0, 0.5, 0.0}

\newcommand{\RNum}[1]{\uppercase\expandafter{\romannumeral #1\relax}}

\newtheorem{example}{Example}
\newtheorem{assumption}{{\bf Assumption}}

\begin{document}

\title{Uniform-in-Time Weak Error
  Analysis for  Stochastic
  Gradient Descent Algorithms via Diffusion Approximation}
\author{
Yuanyuan Feng  \thanks{Department of Mathematics, Penn State University, University Park, PA, 16802, USA (yzf58@psu.edu).}
\and
Tingran Gao \thanks{Committee on Computational and Applied Mathematics, Department of Statistics, University of Chicago, Chicago, IL, 60637, USA (tingrangao@galton.uchicago.edu).}
\and
Lei Li\thanks{School of Mathematical Sciences, Institute of Natural Sciences, MOE-LSC, Shanghai Jiao Tong University, Shanghai, 200240, P. R. China (leili2010@sjtu.edu.cn).}
\and
Jian-Guo Liu \thanks{Department of Mathematics and Department of Physics, Duke University, Durham, NC, 27708 (jliu@phy.duke.edu).}
\and
Yulong Lu \thanks{Department of Mathematics,
  Duke University, Durham, NC, 27708 (yulonglu@math.duke.edu).}}
\date{}
\maketitle

\begin{abstract}
Diffusion approximation provides weak approximation for stochastic gradient descent algorithms in a finite time horizon. In this paper, we introduce new tools motivated by the \emph{backward error analysis} of numerical stochastic differential equations into the theoretical framework of diffusion approximation, extending the validity of the weak
approximation from finite to infinite time horizon. The new techniques developed in this paper enable us to characterize the asymptotic behavior of constant-step-size SGD algorithms near a local minimum around which the objective functions are locally strongly convex, a goal previously unreachable within the diffusion approximation framework. Our analysis builds upon a truncated formal power expansion of the solution of a Kolmogorov equation arising from diffusion approximation, where the main technical ingredient is uniform-in-time bounds controlling the long-term behavior of the expansion coefficient functions near the
local minimum. We expect these new techniques to bring new understanding of the behaviors of SGD near local minimum and greatly expand the range of applicability of diffusion approximation to cover wider and deeper aspects of stochastic optimization algorithms in data science.
\end{abstract}


\section{Introduction}

Stochastic gradient descent (SGD) is a prototypical
stochastic optimization algorithm widely used for solving
large scale data science problems
\cite{RM1985,Zhang2004,SSSS2009,MB2011,SZ2013,BM2013}, not
only for its scalability to large datasets, but also due to
its surprising capability of identifying parameters of deep
neural network models with better generalization behavior
than adaptive gradient methods
\cite{WSCLNMKCGM2016,KS2017,WRSSR2017}. The
past decade has witnessed growing interests in accelerating
this simple yet powerful optimization scheme
\cite{PJ1992,RSS2012,JZ2013,DBL2014,RHSPS2015,GOP2015}, as
well as better understanding its dynamics, through the lens of
either discrete Markov chains \cite{dieuleveut2017,JKNvW2018} or
continuous stochastic differential equations
\cite{lte17,LTE2018,feng2017,hulililiu2018}.

This paper introduces new techniques into the theoretical
framework of \emph{diffusion approximation}, which provides \emph{weak approximation} to SGD algorithms through the
solution of a modified stochastic differential equation (SDE). Though numerous novel insights have been gained from this
continuous perspective, it was previously  unclear whether the modified SDEs can be adopted to study the asymptotic
behavior of SGD, since the weak approximation is only valid over a finite time interval. In the nonconvex case, the approximation error blows up as time goes to infinity. For example, when the coefficient functions are bounded, the SDEs share the behaviors of random walks in high dimension space, which are transient. One will lose control of the system quickly as time goes on.
In the strongly convex case, the problem remains open due to the unbounded diffusivity in the SDEs. We show in this paper that it is possible to study an \emph{approximate solution} of the modified SDE for the latter case, which admits \emph{uniform-in-time}
weak error bounds and can thus be used for investigating the
long-term behavior of SGD dynamics.

We concern ourselves in this paper with the problem of
optimizing an empirical loss function $f:\mathbb{R}^d\rightarrow\mathbb{R}$
\begin{equation}
  \label{eq:empirical-loss}
  f\left(\theta\right) = \frac{1}{N_s}\sum_{i=1}^{N_s}\ell_{\theta} \left(z_i, y_i\right)
\end{equation}
where $\left\{ \left( z_i,y_i \right) \right\}_{i=1}^{N_s}$ are
the training data ($z_i$'s and $y_i$'s are the data and
labels, respectively) and $\ell_{\theta} \left( \cdot,\cdot
\right)$ is the loss function with parameter $\theta$ to be
learned. We will assume local strong convexity for $f$ through
the individual loss functions $\left\{\theta\mapsto
  \ell_{\theta}\left( z_i,y_i \right)\right\}_{i=1}^{N_s}$. The
true gradient of $f$ takes the form
\begin{equation}
  \label{eq:true-grad}
  \nabla f \left( \theta \right) = \frac{1}{N_s}\sum_{i=1}^{N_s}
  \nabla_{\theta}\ell_{\theta} \left(z_i, y_i\right).
\end{equation}
The ``stochastic gradient'' considered in this paper are ``mini-batches'' subsampled from the summands $\left\{
  \nabla_{\theta}\ell_{\theta}\left( z_i,y_i \right) \right\}$ in \eqref{eq:true-grad}, properly normalized so they provide an unbiased
estimate for the true gradient. More specifically, fix a \emph{batch size} parameter $B\in\mathbb{N}$, $1\leq B\leq
N_s$, and let $\xi$ be a subset of $B$ distinct elements
uniformly sampled from the integers $\left\{ 1,\dots,N_s
\right\}$ without replacement, we set
\begin{equation}
  \label{eq:stochastic-grad-defn}
  \nabla f \left( \theta; \xi \right):=\frac{1}{B}\sum_{j\in\xi}\nabla_{\theta}\ell_{\theta}\left( z_j,y_j \right).
\end{equation}
Such constructed stochastic gradients are unbiased estimates
of the true gradient in the sense that
$\mathbb{E}_{\xi}\left[ \nabla f \left(
    \cdot;\xi\right)\right]=\nabla f$.

Below, we will use $x$ to mean the parameter $\theta$ and $X_n$ to mean the discrete iterates in SGD, as is standard in numerical analysis of SDEs. The notation "$\mathbb{E}_x$" will be used to mean expectation under the initial condition $X(0)=x$ for SDE or $X_0=x$ for the SGD iterates. Also, $\Xi$ will be used to denote the set of all possible values of $\xi$, and in the situation described above, it is the set of all subsets of $\{1,2,\ldots, N_s\}$ with size $B$.
The iterative stochastic numerical scheme under consideration throughout
this paper is
\begin{equation}
\label{eq:sgd-dynamics}
  X_{n+1}=X_n-\eta\nabla f\left( X_n;\xi_n \right),\quad n=0,1,\dots
\end{equation}
where $\eta>0$ is the constant step size and $\nabla f(\cdot; \xi_n)$ is the stochastic gradient with $\xi_n\in \Xi$ being i.i.d..
We characterize the asymptotic distributional behavior of
the iterates $\left\{ X_n \right\}_{n\geq 0}$ as $n$
approaches infinity, by adapting tools from \emph{backward
  error analysis} of stochastic numerical schemes
\cite{debussche2012weak,shardlow2006modified,abdulle2012high,abdulle2014high,kopec2014weak,kopec2015weak}
to modified SDEs arising from the diffusion approximation
\cite{lte17,LTE2018,feng2017}. So far, asymptotic
analysis for the dynamics of \eqref{eq:sgd-dynamics} have
been made possible only through the Markov chain techniques
\cite{dieuleveut2017, JKNvW2018}.
We also refer to \cite{SS17, sirignano2017stochastic} for some convergence analysis of stochastic gradient descent methods for continuous time models.
This paper is our first
attempt at fully unleashing the rich and powerful SDE
techniques for studying stochastic numerical optimization
schemes in large scale statistical and machine learning.



\subsection{Main Contribution: Long-Time Weak Approximation for
  SGD via SDE}
\label{sec:related-work-main}

The dynamics of discrete, iterative numerical algorithms can
often be better understood from their continuous time limit,
typically described by ordinary differential equations. This
perspective has been proven fruitful in the analysis of many
deterministic optimization algorithms
\cite{Fiori2005,HM2012,DSE2012,ORXYY2016,SBC2016}. An
analogy of this type of continuous-time-limit analysis for
SGD algorithms is provided by the \emph{diffusion
  approximation} \cite{lte17,feng2017}: in any
\emph{finite} time interval, the distribution of $X_n$ defined by the SGD dynamics
\eqref{eq:sgd-dynamics} is close to the distribution of the
solution of the following SDE at time $t = n\eta$:
\begin{equation}
  \label{eq:modifiedSDE}
  \mathrm{d}X=-\nabla\left[f(X)+\frac{1}{4}\eta \left\|\nabla f(X)\right\|^2\right]\mathrm{d}t+\sqrt{\eta \Sigma(X)}\,\mathrm{d}W,
\end{equation}
where 
\[
\Sigma =\mathbb{E}_{\xi}\left[ \left(\nabla f(\cdot;\xi)-\nabla f\right) \otimes
  \left(\nabla f(\cdot; \xi)-\nabla f\right)\right]
\]
 is the covariance matrix of the
random gradients, and $W$ is the standard Brownian motion
\cite{Oksendal2003}. In numerical SDE literature,
SDE of type \eqref{eq:modifiedSDE} is often referred to as the
\emph{stochastic modified equations}; they play an important
role in constructing high-order numerical approximation
schemes for invariant measures of ergodic SDEs (see, e.g.,
\cite{abdulle2012high, abdulle2014high}). In the context of
data science, diffusion approximation has been used to gain
insights into online PCA~\cite{feng2017},
entropy-SGD~\cite{chaudhari2016entropy,chaudhari2018deep},
and nonconvex optimization~\cite{hulililiu2018}, to name just a
few.

Despite its effectiveness as a continuous analogy of
stochastic numerical optimization algorithms, the range of
applicability of diffusion approximation is significantly
limited by its restricted validity in a finite time
interval. In particular, this means that the solution of the
SDE \eqref{eq:modifiedSDE} can be used to rigorously
approximate only a finite number (though very large) of SGD
iterates \eqref{eq:sgd-dynamics}, and thus can not be used
in the same way as Markov-chain-based theoretical analysis
\cite{dieuleveut2017,JKKNPS2017,MHB2017} to study the
asymptotic behavior of $\left\{ X_n \right\}_{n\geq 0}$ as
$n\rightarrow\infty$. This paper aims at closing this
theoretical gap by extending the validity of diffusion
approximation from finite- to infinite-time horizon. To the
best of our knowledge, this is the first work that studies
the asymptotic distributional behavior of SGD from an SDE
perspective.

Our main technical contribution in this paper is to adopt
the framework of \emph{weak backward error analysis} to the
solution $u=u \left( x,t \right)=\mathbb{E}_x \left[ \varphi\left( X(t) \right) \right]$ of the following backward Kolmogorov equation associated with
SDE \eqref{eq:modifiedSDE}:
\begin{equation}
  \label{eq:backwardKol}
  \begin{aligned}
    &\frac{\partial u}{\partial t}=-\nabla
    f \cdot \nabla
    u +\eta\left(-\frac{1}{4}\nabla\left\|\nabla
        f \right\|^2\cdot\nabla
      u +\frac{1}{2}\mathrm{Tr}\left(\Sigma \nabla^2u \right)\right)\\
    &u \left( x,0 \right)=\varphi \left( x \right)
  \end{aligned}
\end{equation}
where we recall that $\mathbb{E}_x$ stands for taking expectation under the initial condition $X(0)=x$, $\mathrm{Tr}\left( A \right)$ stands for the trace of
a square matrix $A$, $\Sigma = \Sigma(x)$ is the covariance matrix as in
\eqref{eq:modifiedSDE}, and $\nabla u$, $\nabla^2 u$ denote
the gradient and Hessian of $u=u \left( x,t \right)$ with
respect to the spatial variable $x$. The function
$\varphi:\mathbb{R}^d\rightarrow\mathbb{R}$ is an arbitrary
``observable'' of the stochastic dynamical system that
characterizes properties of interest of the iterates $\left\{ X_n
\right\}_{n\geq 0}$. Weak error analysis concerns the
behavior of $\left\{\varphi \left( X_n
  \right)\right\}_{n\geq 0}$ for any $\varphi$ with
sufficient regularity; for instance, by taking $\varphi=f$,
we can study the asymptotic oscillatory and/or concentration
behavior of the objective values $f \left( X_n \right)$ with
respect to the global minimum if standard convexity
assumptions are imposed on $f$.

In a nutshell, backward error
analysis is based on identifying the associated generator of a
numerical scheme with the generator of a modified SDE,
up to higher order terms in the powers of the step size
$\eta$. This can be achieved, e.g., by formally expanding
the generator of the modified SDE into a power series of the
step size, and then determining the coefficients (which are
functions of the space and time variables, but not the step
size $\eta$) of this power series using information from the
numerical scheme; it is then natural to expect that a proper
truncation of this formal power series can be used as a
reasonable approximation for the iterates of the stochastic
numerical scheme (in the weak sense), even though the formal
series may not converge (and thus the solution of the SDE
may not be a good approximation for the discrete iterates
for all time). As illustrated by many examples in the
numerical analysis of ergodic SDEs (see, e.g.,
\cite{debussche2012weak,shardlow2006modified,abdulle2012high,abdulle2014high,kopec2014weak,kopec2015weak}
and the references therein), it turns out that the
coefficient functions of the formal power series
capture---in a uniform-in-time fashion---the leading order
behavior of the discrete numerical scheme; this enables
practitioners to draw conclusion on the closeness between
the invariant measure of the numerical scheme and the
invariant measure of the truncated formal series. In other
words, though solutions of \eqref{eq:backwardKol} can not be
used directly to capture the long-term behavior of SGD
\eqref{eq:sgd-dynamics}, we construct an alternative,
auxiliary function approximation of the solution of
\eqref{eq:backwardKol}, which turns out to be a superior
weak approximation of \eqref{eq:sgd-dynamics} in the sense
that the approximation error is uniform-in-time and in
higher powers of the step size $\eta$. The time-uniformity
of such a truncated formal series approximation enables us
to study the asymptotic distributional behavior of the
iterates of \eqref{eq:sgd-dynamics}, thus closing the gap in
the theoretical analysis between diffusion approximation and
Markov-chain-based analysis. We provide an overview for the
main steps in our analysis in the next section.


\subsection{Sketch of the Main Approach}
\label{sec:sketch-main-approach}

We consider a formal expansion of the solution $u=u \left( x,t
\right)=\mathbb{E}_x \left[ \varphi \left( X(t) \right) \right]$ of
\eqref{eq:backwardKol} in a power series with respect to the
step size $\eta>0$:
\begin{equation}
\label{eq:asyexp}
u \left( x,t \right)=\sum_{\ell=0}^{\infty}\eta^{\ell} u_{\ell} \left( x,t \right).
\end{equation}
For the ease of exposition, let us introduce short-hand
notations $\cL_1,\cL_2$ for the differential operators
appearing in the right hand side of \eqref{eq:backwardKol}:
\begin{equation}
  \label{eq:differential-ops}
  \cL_1:=-\nabla f\cdot \nabla,\qquad \cL_2:=-\frac{1}{4}\nabla\left\|\nabla
        f \right\|^2\cdot\nabla
      +\frac{1}{2}\mathrm{Tr}\left(\Sigma \nabla^2\right)
\end{equation}
with which \eqref{eq:backwardKol} can be recast into
\begin{equation}
  \label{eq:backwardKolRecast}
  \begin{aligned}
    &\partial_t u=\cL_1u +\eta\cL_2 u,\\
    &u \left( x,0 \right)=\varphi \left( x \right).
  \end{aligned}
\end{equation}
Formally plugging \eqref{eq:asyexp} into
\eqref{eq:backwardKolRecast} and equating terms
corresponding to the same powers of $\eta$, we can determine
all coefficient functions $u_n\left( x,t \right)$ from
solving corresponding PDEs, namely, for $\ell=0$
\begin{equation}
\label{eq:u0-ode}
  \begin{aligned}
    &\partial_t u_0=\cL_1u_0,\\
    & u_0(x,0)=\varphi(x)
  \end{aligned}
\end{equation}
and for $\ell\geq 1$
\begin{equation}
\label{eq:uell-ode}
  \begin{aligned}
    &\partial_t u_{\ell}=\cL_1u_{\ell}+\cL_2u_{\ell-1},\\
    & u_{\ell}(x,0)=0.
  \end{aligned}
\end{equation}
Determining any $u_{\ell}$ can thus be done by inductively
solving a sequence of first-order PDEs \eqref{eq:u0-ode}
\eqref{eq:uell-ode}. In fact, with some work we can
establish exponential convergence of each $u_{\ell}$ to its
equilibrium state as $t$ approaches infinity, provided that
$f$ is strongly convex.

We then construct an approximation for $u$ by truncating the
formal series \eqref{eq:asyexp}, yielding
\begin{equation}
\label{eq:truncatedsum}
u^N \left( x,t \right)=\sum_{\ell=0}^{N}\eta^{\ell} u_{\ell} \left( x,t \right).
\end{equation}
If the formal series \eqref{eq:asyexp} converges uniformly,
$u^N$ is certainly a good approximation of $u$ up to an
order $\mathcal{O}\left( \eta^{N+1} \right)$ error. The crux
of our argument is that, even when the convergence of
\eqref{eq:asyexp} is not guaranteed, it turns out that we
can still use $\left\{u^1 \left(
    x,n\eta\right)\right\}_{n\geq 0}$ as good
approximation for $\left\{\mathbb{E}_x\left[ \varphi \left(
      X_n \right) \right]\right\}_{n\geq 0}$ (recall that $\mathbb{E}_x$ represents the expectation conditioned on the initial condition $X_0=x$); most notably,
the $\mathcal{O}\left( \eta^2 \right)$ error in this
approximation is bounded uniformly in $n$, allowing us to
draw quantitative conclusions on the asymptotic
distributional behavior of $\mathbb{E}_x\left[ \varphi \left(
      X_n \right) \right]$ from that of $u^1 \left(
    x,n\eta\right)$. Since $u^1$ corresponds to a measure $\nu^1$ independent of the test function $\varphi$, our argument then justifies that the measure $\nu^1$ approximates the distribution of the SGD with second order weak accuracy. It is very tempting to push this idea further by considering $u^N$, $N>1$ in place of $u^1$ and
  expecting it to better approximate $\mathbb{E}_x\left[ \varphi \left(X_n \right) \right]$ up to higher orders of error; however,
  our analysis indicates that in general $\left| u^N \left( x,n\eta
    \right) - \mathbb{E}_x \left[ \varphi \left( X_n \right)
    \right] \right|=\mathcal{O}\left( \eta^2 \right)$ can no
  longer be improved by choosing $N>1$, even
  though $u^N$ could be a better approximation for the
  solution $u$ of the backward Kolmogorov equation
  \eqref{eq:backwardKol} when $N>1$.

The superior, uniform-in-time approximation of the truncated formal expansion to
$\mathbb{E}_x \left[ \varphi \left( X_n \right) \right]$ is achieved by the fact that the
coefficient functions
$u_{\ell}$ are totally determined by the local behavior of
$f$ and $\varphi$ (i.e. behaviors on compact sets), whereas the solution $u$ of
\eqref{eq:backwardKol} depends on the global information and
is thus harder to control. Due to this locality, the local strong convexity of $f$ then leads to the exponential decay of the derivatives for the coefficient functions $u_{\ell}$, which finally gives the uniform-in-time weak approximation. This will become transparent
after we establish Theorem~\ref{p:main}. The locality can be illustrated by a
toy SDE example in one dimension  with $f \left( x
\right)=\frac{1}{2}x^2$, and $\Sigma \left( x \right)\equiv
1$. Note that this SDE example is simply given to illustrate the roles of $u_{\ell}$ and why they are local, while it is not necessarily the diffusion approximation of some SGD iteration. In this example, SDE \eqref{eq:modifiedSDE} corresponds
to an Ornstein--Uhlenbeck process, and the solution of
\eqref{eq:backwardKol} adopts the explicit integral representation
\begin{equation}
\label{eq:ou-example}
  \begin{aligned}
    u(x, t)&=\frac{1}{\sqrt{2\pi
        S}}\int_{\mathbb{R}^d}\varphi(w)\exp\left(-\frac{(w-xe^{-(1+2\eta)t})^2}{2S}\right)\,\mathrm{d}w\\
    &=\frac{1}{\sqrt{2\pi}}\int_{\mathbb{R}^d} \varphi\left(xe^{-(1+2\eta)t}+\sqrt{S}y\right)\exp(-y^2/2)\,\mathrm{d}y
  \end{aligned}
\end{equation}
where
\begin{equation*}
  S=\frac{\eta}{2(1+2\eta)}\left(1-e^{-2(1+2\eta)t}\right).
\end{equation*}
We can obtain a formal expansion of $u \left( x,t \right)$
in terms of $\eta$ using a Taylor expansion
for $\varphi$ at $xe^{-(1+2\eta)t}$ in the integrand of
\eqref{eq:ou-example}. We keep $2m$ terms in the Taylor
expansion and note that all odd powers of $\sqrt{S}$
vanish, which leads to the following expansion of error
$\mathcal{O}\left( \eta^{m+1} \right)$:
\begin{equation*}
  u(x, t)=\sum_{k=0}^{m}\frac{1}{(2k)!}\varphi^{(2k)}\left(xe^{-(1+2\eta)t}\right)S^k\cdot\frac{1}{\sqrt{2\pi}}\int_{\mathbb{R}^d} y^{2k}\exp(-y^2/2)\,dy+O\left(\eta^{m+1}\right).
\end{equation*}
The $u_{\ell}$'s can then be obtained by further expanding the
functions about $\eta=0$ and combining terms of equal
powers. Clearly, such obtained $u_{\ell}$'s in this expansion will only
depend on the derivatives of $\varphi$ at $xe^{-t}$; meaning that $u_{\ell}(x,t)$ only depends on the behaviors of $\varphi$ inside the ball with radius $|x|$, whereas for any $x$, $u(x,t)$ depends on the values of $\varphi$ in the whole space.  The formal series expansion is like the Taylor series of the function $u(x,t)$ with respect to $\eta$. As known, in general one can not expect the Taylor series to converge to the original function unless the function is analytic, which exactly resembles the
difference between the solution of \eqref{eq:backwardKol}
and the truncated formal series expansion \eqref{eq:truncatedsum}: the latter maintains only the barely minimum local information in the diffusion approximation for characterizing the asymptotic distributional behavior of the dynamics of SGD
\eqref{eq:sgd-dynamics}.

Full details of our theoretical framework can be found in Section~\ref{sec:main-results} and the appendices.


\subsection{Outline}
\label{sec:outline}

In the remainder of this paper, we present our main theorems
and main proofs in Section~\ref{sec:main-results}, and validate
our theory with numerical experiments in
Section~\ref{sec:numer-exper}. Technical lemmas and
auxiliary results are deferred to the appendices. We
conclude this paper and propose future directions in Section~\ref{sec:conclusion}.

\section{Main Results}
\label{sec:main-results}

We begin by stating the assumption that will be used throughout this paper (recall that $\Xi$ is the set of all possible values of the random parameters $\xi$).
\begin{assumption}
\label{ass:strongconv}
Without loss of generality, assume $f$ has a local minimum at the origin $x_{*}=0$. Gradients of the random functions $\left\{f(\cdot;\xi)\in
  C^3(\mathbb{R}^d)\mid \xi\in \Xi\right\}$ provide unbiased estimates for the gradient of $f$, i.e.,
$\mathbb{E}_{\xi}\left[ \nabla f\left( x;\xi \right) \right]=\nabla f \left( x \right)$ for all
$x\in\mathbb{R}^d$. Moreover, we assume the following hold
for the random functions. There exists $R_1>0$ such that
\begin{enumerate}[(1)]
\item\label{item:1} Each random function $f \left( \cdot;\xi \right)$ is
  $\gamma$-strongly convex in $B(x_{*}, R_1)$, i.e., $f(\cdot;
  \xi)-\frac{1}{2}\gamma \|\cdot\|^2$ is convex for all $\xi\in \Xi$;
\item\label{item:2} The random gradients at $x_*=0$ are uniformly
  bounded:
\begin{align}\label{e:f1}
\sup_{\xi}\|\nabla f(0;\xi)\|\le b<\infty.
\end{align}
for some $b>0$ and more over 
\begin{gather}
R_1>\frac{16 b}{3\gamma}=:R_0.
\end{gather}
\end{enumerate}
\end{assumption}

Though our assumption on the individual $f \left( \cdot;\xi \right)$'s appears
to be strong, it is not particularly restrictive for the
most commonly encountered scenario of SGD application where
each random function $f \left( \cdot;\xi \right)$ is constructed from the same
loss function $loss \left( y_{\xi},g \left( z_{\xi} \right)
\right)\equiv \ell_{\theta}(z_{\xi}, y_{\xi})$, and the only source of randomness is in the random
data $\left( z_{\xi},y_{\xi} \right)$ sampled from an
unknown data distribution. In this case,
Assumption~\ref{ass:strongconv} can be stated just once for
the loss function, as done in \cite{LSLS2018}. Such an
assumption on the individual summands in the empirical loss
function has also appeared previously in Markov-chain-based
studies of stochastic gradient descent algorithms,
e.g. Assumption~A4 in \cite{dieuleveut2017}. The boundedness
assumption \eqref{e:f1} is obviously satisfied if the loss
function $\ell_{\theta}\left( z_i, y_i \right)$ is bounded at $\theta=0$ for all data $(z_i, y_i)$.

In the remainder of this section, we divide our exposition
of the main results into two parts. Estimates establishing
the exponential convergence of the coefficient functions
of the formal series expansion appear in
Section~\ref{sec:form-series-expans}, and their applications
to studying the asymptotic distributional behavior of SGD
iterates appear in Section~\ref{sec:dynamics-sgd-with}.

\subsection{Formal Series Expansion}
\label{sec:form-series-expans}

Under the local strong convexity assumption in
Assumption~\ref{ass:strongconv}, the following two lemmas
can be easily established. We defer the proofs to
Appendix~\ref{sec:technical-lemmas}.
In particular, the convergence in
Wasserstein-$2$ distance in Lemma~\ref{lmm:w2conv} is
well-known (see, e.g., Proposition~1 in \cite{dieuleveut2017}); we contain a simple proof in
Appendix~\ref{sec:technical-lemmas} for completeness. In the
rest of this paper, for any $R>0$, we denote $B \left( 0,R
\right)$ for the Euclidean ball of radius $R$ centered at
the origin (which is also the global minimum of $f$ by
Assumption~\ref{ass:strongconv}).

\begin{lemma}\label{lmm:uniformbound}
Suppose Assumption \ref{ass:strongconv} holds, and denote
$R_0=16b/3\gamma$. If $R\in (R_0, R_1]$, set \[
\eta_0=\min\left\{\frac{1}{2\gamma}, \frac{3R}{8b},\frac{3\gamma R^2/8-2bR}{2\gamma bR+b^2}\right\}.
\]
Then for any $\eta\leq \eta_0$ and $X_0=x\in B(0,R)$, we
have $X_n\in B(0,R)$ for all $n\geq 0$. In other words, under these
assumptions the sequence generated by the SGD is uniformly
bounded in both $n$ and $\xi$.
\end{lemma}

\begin{lemma}
\label{lmm:w2conv}
Suppose Assumption~\ref{ass:strongconv} holds, and let
$\mu_n$ denote the law of the $n^{\textrm{th}}$ iterate $X_n$ of SGD
\eqref{eq:sgd-dynamics}. Assume $\supp \mu_0\subset B(0, R)$ with $R\in (R_0, R_1]$
and denote $L=\sup_{\xi}\sup_{\|x\|\le R}\|\nabla^2
f(x; \xi)\|$,  where $\left\| \nabla^2 f\left( x; \xi \right) \right\|$ is
the spectral norm (largest singular value) of the Hessian matrix $\nabla^2f
\left( x; \xi \right)$. Then, when $\eta$ is sufficiently small,
$\mu_n$ converges to a probability measure $\pi$ under the
Wasserstein-$2$ norm ($W_2$-norm) at exponential rate
\[
W_2(\mu_n, \pi)\le C\rho^n
\]
for $\rho=(1-2\gamma\eta+\eta^2L)^{1/2}$.
\end{lemma}

\begin{remark}\label{rmk:twomin}
Clearly, for different local minima around which the loss functions are locally strongly convex, the probability measure $\pi$
will be different. Since the SDEs in diffusion approximation has nonzero transition probability connecting any two points in space,
the diffusion approximation cannot be uniform in time for such {\it globally nonconvex} cases. Even for such globally nonconvex loss functions, our theory indicates that the local information of diffusion approximation is enough to capture the long time behavior of SGD near the local minimum. To obtain global diffusion approximation for such nonconvex cases, one has to modify the values of the loss function outside the region where SGD can see.
\end{remark}

We define the $S^{(n)}$ operator by
\begin{equation}
(S^{(n)}\varphi)(x):=\mathbb{E}_x\left[\varphi(X_n)\right]=\int_{\mathbb{R}^d} \varphi(y)\,\mu_n(\mathrm{d}y).
\end{equation}
Fixing any smooth test function $\varphi$, we denote
\begin{equation}
U^n(x):=S^{(n)}\varphi(x).
\end{equation}
We know from \cite{feng2017} that $S$ is $L^{\infty}$-nonexpansive, and that $\{S^{(n)}\}$ is a semigroup generated by $S$ such that
\begin{equation}
S^{(n)}=S^n:=S\circ S\ldots \circ S \text{~($n$ copies)}.
\end{equation} 

Since convergence in Wasserstein distance implies weak
convergence, Lemma~\ref{lmm:w2conv} implies
\begin{equation}
  \lim_{n\to\infty}U^n=\int_{\mathbb{R}^d} \varphi\,\mathrm{d}\pi .
\end{equation}
However, this does not provide much precise and/or
quantitative information regarding how $U^n$
converges to $\int_{\mathbb{R}^d}
\varphi\,\mathrm{d}\pi$. An important goal of this paper is
to shed new lights on the dynamics of $\mu_n$ as
$n\rightarrow\infty$. Within the diffusion approximation
framework, it can be shown (see, e.g., \cite{feng2017})
that the semi-group evolution $U^n$ admits a \emph{weak
  second order diffusion   approximation} over a finite time
interval $\left[ 0,T \right]$, in the sense that for all
sufficiently smooth $\varphi$ there holds
\begin{equation}
\label{eq:diffusion-approx-prev}
  \sup_{n\le T/\eta}\|U^n \left( \cdot \right)-u(\cdot, n\eta)\|_{L^{\infty}}\le C(T, \varphi,\eta_0)\eta^2
\end{equation}
for all $\eta\le \eta_0$, where $\eta_0>0$ is a constant,
and $u \left( x,t \right)=\mathbb{E}_x\left[ \varphi \left(
    X(t) \right) \right]$ is the solution of the backward
Kolmogorov equation \eqref{eq:backwardKol}. Roughly
speaking, SDE \eqref{eq:modifiedSDE} can be regarded
as the weak approximation of the SGD \eqref{eq:sgd-dynamics}
over any finite time interval $[0, T]$. Unfortunately, the
validity of this approximation for infinite time ($T\rightarrow\infty$) is unclear. For nonconvex objective functions, it is known that the approximation can break down quickly as $T\rightarrow\infty$. One obvious example is the situation described in Remark \ref{rmk:twomin}. For globally and strongly convex objective functions (which generate confining dynamics for SGD, according to Lemma~\ref{lmm:uniformbound}), the
validity of long time diffusion approximation is still in doubt due to the unboundedness diffusivity encoded in $\Sigma$.
As motivated in Section~\ref{sec:sketch-main-approach}, we will switch gears and use a truncated formal series
\eqref{eq:truncatedsum} in place of the solution $u$ of \eqref{eq:backwardKol} to approximate $U^n$, for all
arbitrarily large $n\geq 0$.


Before stating the main technical result concerning the
exponential convergence of the $u_{\ell}$'s in the formal
asymptotic expansion, we introduce another notation to
simplify the exposition and proof: denote
\[
I_{k}=\{J=(j_1,j_2,\ldots, j_k): 1\le j_k \le d\}.
\]
For $J\in I_k$, we denote
\[
\partial^Ju:=\partial_{j_1}\ldots\partial_{j_k}u.
\]
We write $J_0\leq J$ if  $\partial ^J u$ is a partial
derivative of $\partial^{J_0} u $, and $J_1=J-J_0$ if
$\partial^J=\partial^{J_0}\partial^{J_1}$.

\begin{remark}\label{r:multin}
The reason that we adopt the notation $\partial^J$ instead
of the standard multi-index notation $\partial^{\alpha}$
where $\alpha=(\alpha_1,\ldots,\alpha_d)$ with
$\alpha_1+\ldots+\alpha_d=k$ is mainly for the sake of
clarity and simplicity of exposition. First, this convention
is widely used for tensor analysis in physics and
engineering. More importantly, in
Appendix~\ref{sec:proof-main-technical-thm} where we prove
Theorem~\ref{p:main}, $\sum_{J\in
  I_{n+1}}\partial_t(\partial^J u)^2$ naturally has a
quadratic form associated with the Hessian matrix
$\nabla^2f$ so that we can make use of the strong
convexity. If we use $\partial^{\alpha}$ notation, we will have to multiply some
weight factors $w_{\alpha}$ such that
$\sum_{|\alpha|=k}w_{\alpha}\partial_t(\partial^{\alpha}u)^2$
has the desired quadratic form.
\end{remark}

We are now ready to present our main estimates for the
exponential rate of decay for the coefficient functions in
the formal series expansion \eqref{eq:asyexp}. We will use $P$ to denote a generic polynomial whose concrete form may change from line to line. The number of arguments for the polynomials will also be clear in the context, which we will not emphasize.
\begin{theorem}
\label{p:main}
Assume Assumption~\ref{ass:strongconv} holds, $\eta\leq
\eta_0$ and $R \in (R_0, R_1]$, for $R_0>0$, $R_1>0$ and $\eta_0>0$ defined as in Lemma~\ref{lmm:uniformbound}. Recall that $x^*=0$ is the unique minimum of $f$.

\begin{enumerate}[(i)]
\item For
an arbitrary test function $\varphi\in C^1(\mathbb{R}^d)$, $u_0$ satisfies
\begin{gather}
\sup_{x\in B(0, R)}|u_0(x,t)-\varphi(0)|\le R\|\varphi\|_{C^1(B(0, R))}e^{-\gamma t}.
\end{gather}
In addition, if $\varphi\in C^k \left( B \left( 0,R \right)
\right)$ and $f\in C^{k+1}\left( B \left( 0,R \right)
\right)$ for some integer $k\ge 1$, then 
\begin{gather}\label{e:u0deri}
\sup_{J\in I_k} \sup_{x\in B(0,R)}|\partial^Ju_0(x,t)|\leq P\Big(\|\varphi\|_{C^k(B(0,R))}, \|f\|_{C^{k+1}(B(0,R))}\Big)e^{-\gamma t}. 
\end{gather}

\item
For any $n\ge 1$, if the test function $\varphi\in C^{2n+1}
\left( B \left( 0,R \right) \right)$ and $f\in
C^{2n+2}\left( B \left( 0,R \right) \right)$, then for any $\gamma'< \gamma$,
\begin{gather}
\sup_{x\in B(0, R)}|u_n(x,t)-\varphi_n|\le P\Big( \|\varphi\|_{C^{2n+1}(B(0,R))},\|f\|_{C^{2n+2}(B(0,R))}, \|\Sigma\|_{C^{2n-1}(B(0,R))}\Big) e^{-\gamma' t},
\end{gather}
where
\begin{gather}
\varphi_n:=\int_0^{\infty}\cL_2 u_{n-1}(0, s)\,ds.
\end{gather}

In addition, if $\varphi\in C^{k+2n} \left( B \left( 0,R \right)
\right)$ and $f\in C^{k+1+2n}\left( B \left( 0,R \right)
\right)$ for some $k\ge 1$, then for any $\gamma''< \gamma$,
\begin{equation}\label{e:underi}
\sup_{J\in I_k} \sup_{x\in B(0,R)}|\partial^Ju_n(x,t)|
\leq P\Big(\|\varphi\|_{C^{k+2n}(B(0,R))}, \|f\|_{C^{k+1+2n}(B(0,R))}, \|\Sigma\|_{C^{k+2n-2}(B(0,R))}\Big)e^{-\gamma''t}.
\end{equation}
\end{enumerate}

\end{theorem}

The proof of Theorem~\ref{p:main} is quite technical; we
defer full details to
Appendix~\ref{sec:proof-main-technical-thm}. We state an
immediate corollary of Theorem~\ref{p:main} to close this
subsection.

\begin{corollary}\label{eq:corollaryapprox}
Under the same assumptions as in Theorem~\ref{p:main}, the
truncated formal series $u^N$ defined in 
\eqref{eq:truncatedsum} ``approximately satisfies'' the
backward equation \eqref{eq:backwardKol} in the sense that
\begin{gather}
\partial_tu^N=(\cL_1+\eta \cL_2)u^N-\eta^{N+1}\cL_2u_N\,.
\end{gather}
Consequently, if $\varphi\in C^{2k+2N}\left( B \left( 0,R
  \right) \right)$ and $f\in C^{2k+1+2N}\left( B \left( 0,R
  \right) \right)$ for some $k\ge 1$, we have
\begin{gather}
\sup_{x\in B(0, R)}|\partial_t^ku^N-(\cL_1+\eta \cL_2)^ku^N|
\le C(N,R)e^{-\gamma t/2 }\eta^{N+1}
\end{gather}
where $C(N,R)
=Q_{N,k}\Big(\|\varphi\|_{C^{2k+2N}(B(0,R))},\|f\|_{C^{2k+1+2N}(B(0,R))},\|\Sigma\|_{C^{2k}(B(0,R))}\Big)$
for some polynomial $Q_{N,k}$.
\end{corollary}

It is clear from Theorem~\ref{p:main} that all the
coefficient functions $u_n(x, t)$ depend only on the
information of $f$ and $\Sigma$ inside the ball $B(0,
\left\|x\right\|)$, in the sense that the bound does not
change if we modify the values of $\varphi$, $f$, and $a$
outside $B(0, \left\|x\right\|)$. Thus $u_n$ reflects the
``local information'' of $u$. This is in stark contrast with the solution of
\eqref{eq:backwardKol} at $x$, which inevitably depends on the values
of $\varphi$ outside $B(0, \left\|x\right\|)$ due to the
parabolicity of the second order PDE
\eqref{eq:backwardKol}. As explained in
Section~\ref{sec:sketch-main-approach}, this is due to the
fact that the $u_n \left( x,t \right)$'s are essentially the
``Taylor expansion coefficients'' of $u$ with respect to the
step size. This is also the reason that we referred to
\eqref{eq:asyexp} as only a formal series expansion: in
general the Taylor series needs not converge to the original
function. See also the Ornstein--Uhlenbeck process example
in Section~\ref{sec:sketch-main-approach} for a concrete example.

\subsection{Dynamics of SGD with Constant Step Size}
\label{sec:dynamics-sgd-with}

In this subsection we apply the results from
Section~\ref{sec:form-series-expans} to studying the
asymptotic distributional behavior of the SGD dynamics
\eqref{eq:sgd-dynamics}. Throughout the rest of this
subsection, we always assume that  $X_0\in B(0, R)$ and $R$
satisfies the condition of Lemma~\ref{lmm:uniformbound}. The
confining nature of the dynamics allows us to choose very
general functions as test functions, e.g., smooth functions
that grow exponentially as $\|x\|\to\infty$, for the weak
approximation results to hold. This is because we can always
modify the part of the test function outside of $B
(0,R)$. More precisely, we have
\begin{lemma}\label{lmm:localchange}
Under Assumption~\ref{ass:strongconv}, given any test
function $\varphi\in C^k(\mathbb{R}^d)$ for some
$k\in\mathbb{N}$, we can choose $\tilde{\varphi}\in C^k
\left( \mathbb{R}^d \right)$ compactly supported such that
\[
\|\tilde{\varphi}\|_{C^k(\mathbb{R}^d)}\le C\|\varphi\|_{C^k(B(0,R))}
\]
and
\[
\mathbb{E}_x\left[\varphi(X_n)\right]=\mathbb{E}_x\left[\tilde{\varphi}(X_n)\right],~\forall x\in B(0, R).
\]
Similarly, in the formal series expansion \eqref{eq:asyexp}
for the diffusion approximation, replacing $\varphi$ with
$\tilde{\varphi}$ does not change any of the coefficient
functions $u_{\ell}(x, t), x\in B(0, R), \ell\ge 0$.
\end{lemma}
Lemma~\ref{lmm:localchange} is a simple consequence of
transport equations \eqref{eq:u0-ode}
\eqref{eq:uell-ode}. Notably, we emphasize again that the
locality of the coefficient functions $u_{\ell}\left( x,t
\right)$ is in stark contrast with the solution of the
backward Kolmogorov equation \eqref{eq:backwardKol}, since
\eqref{eq:backwardKol} has diffusion effects which is
global. Lemma~\ref{lmm:localchange} indicates we can focus
on test functions compactly supported near the local minimum we care about. The main result of this paper is the following.

\begin{theorem}\label{thm:sgdasym}
Assume Assumption~\ref{ass:strongconv} holds, $\eta\leq
\eta_0$ and $R \in (R_0, R_1]$, for $R_0>0$, $R_1>0$ and $\eta_0>0$ defined as in Lemma~\ref{lmm:uniformbound}.
If $f(\cdot; \xi)\in C^7
\left( B \left( 0,R \right) \right)$ and
$\varphi\in C^6 \left( B \left( 0,R \right) \right)$,
then $u^1=u_0+\eta u_1$ approximates the dynamics of SGD
\eqref{eq:sgd-dynamics} \emph{with weak second order}, in the
sense that there exists a positive constant $C \left(
  \varphi,f,R \right)$ independent of $n$ such that
\begin{gather}
\sup_{x\in B(0, R)}|\mathbb{E}_x \left[ \varphi \left( X_n \right) \right]-u^1(x, n\eta)|\le C(\varphi, f,R) \eta^2.
\end{gather}
\end{theorem}

\begin{proof}
By Lemma~\ref{lmm:localchange}, we can assume without loss
of generality that $\varphi$ is compactly supported and
$\|\varphi\|_{C^k(\mathbb{R}^d)}\le
C_k\|\varphi\|_{C^k(B(0,R))}$ for sufficiently large
$k$. Let us recall the notation $U^n(x)=\mathbb{E}_x \left[ \varphi \left( X_n
  \right) \right]$ and that $S:L^{\infty} \left( \mathbb{R}^d \right)\rightarrow
L^{\infty} \left( \mathbb{R}^d \right)$ forms the semi-group $S^{(n)}=S^n$. Thus,
\begin{equation*}
  U^{n+1}(x)=\mathbb{E}(U^n(x-\eta \nabla f(x; \xi))):=SU^n(x).
\end{equation*}
Noticing that $U^n(x)=S^n\varphi(x)$ and
$\varphi(x)=u^N(x, 0)$, by a telescoping sum we have
\[
U^n(x)-u^N(x, n\eta)=\sum_{j=1}^nS^{n-j}(Su^N(x,(j-1)\eta)-u^N(x, j\eta)).
\]
By the fact that $S$ is $L^{\infty}$ nonexpansive,
\begin{equation}
\label{eq:telescoping-sum}
|U^n(x)-u^N(x, n\eta)|\le \sum_{j=1}^n \|Su^N(x,(j-1)\eta)-u^N(x, j\eta)\|_{L^{\infty}}.
\end{equation}

We fix $N=1$ and for the sake of convenience, we introduce
\begin{gather}
t_j:=j\eta.
\end{gather} 
By Corollary~\ref{eq:corollaryapprox}, it holds for $t\in [t_{j-1}, t_j]$ that
\begin{gather}
u^1(x,t)=u^1(x,t_{j-1})+\int_{t_{j-1}}^{t}(\cL_1+\eta \cL_2)u^1(x,s)\,ds-\eta^2\int_{t_{j-1}}^t\cL_2u_1(x,s)\,ds.
\end{gather}
Substituting this expression of $u^1$ into the right hand side (and repeatedly for some terms), one has
\begin{multline}
u^1(x,t)=u^1(x,t_{j-1})+(t-t^n)\cL_1u^1(x, t_{j-1})+\eta(t-t^n)\cL_2u^1(x,t_{j-1})+\frac{1}{2}(t-t^n)^2\cL_1^2u^1(x, t_{j-1})\\
+\eta\int_{t_{j-1}}^t\int_{t_{j-1}}^s(\cL_2(\cL_1+\eta \cL_2)+\cL_1\cL_2)u^1(x,\tau)\,d\tau ds
+\int_{t_{j-1}}^t\int_{t_{j-1}}^s\int_{t_{j-1}}^{\tau}\cL_1^2(\cL_1+\eta \cL_2)u^1\,dzd\tau ds\\
-\eta^2\int_{t_{j-1}}^t\cL_2u_1\,ds
-\eta^2\int_{t_{j-1}}^t\int_{t_{j-1}}^s(\cL_1+\eta\cL_2)\cL_2 u_1\,d\tau ds
-\eta^2\int_{t_{j-1}}^t\int_{t_{j-1}}^s\int_{t_{j-1}}^{\tau}\cL_1^2\cL_2u_1\,dzd\tau ds.
\end{multline}

Hence,
\begin{equation}
  \begin{aligned}
    \Big|u^1(x, j\eta) &-u^1(x, (j-1)\eta)-\eta(\cL_1+\eta \cL_2)u^1(x, (j-1)\eta) \\
    &-\frac{\eta^2}{2}\cL_1^2u^1(x, (j-1)\eta)\Big|
 \leq C\, \eta^3\sup_{t\in[t_{j-1}, t_j]}\left(\sum_{I=1}^4\sup_{x\in B(0,R)}(|\partial^I u^1|+|\partial^I u_1|)\right)\,.
  \end{aligned}
\end{equation}
By Theorem~\ref{p:main}, 
\begin{equation}\label{eq:onehand}
  \begin{aligned}
    \Big|u^1(x, j\eta) &-u^1(x, (j-1)\eta)-\eta(\cL_1+\eta \cL_2)u^1(x, (j-1)\eta) \\
    &-\frac{\eta^2}{2}\cL_1^2u^1(x, (j-1)\eta)\Big|
 \leq C(f,\varphi,R)\, \eta^3e^{-\gamma(j-1)\eta/2}.
  \end{aligned}
\end{equation}

In the meanwhile, applying Taylor expansion to $Su^1(x,
(j-1)\eta)=\mathbb E\left[u^1(x-\eta\nabla f(x,\xi)),(j-1)\eta\right]$
 and applying Theorem~\ref{p:main} gives
\begin{equation}
\label{eq:theotherhand}
  \begin{aligned}
    \Big|Su^1(x, (j-1)\eta)&-u^1(x, (j-1)\eta)-\eta(\cL_1+\eta \cL_2)u^1(x, (j-1)\eta)\\
    &-\frac{\eta^2}{2}\cL_1  u^1(x, (j-1)\eta)\Big|\leq C(f,\varphi,R)\eta^3e^{-\gamma(j-1)\eta/2}.
  \end{aligned}
\end{equation}
Combining \eqref{eq:onehand} and \eqref{eq:theotherhand}, we have
\[
|Su^1(x,(j-1)\eta)-u^1(x, j\eta)|\le C(f,\varphi,R)\eta^3 e^{-\gamma(j-1)\eta/2}
\]
and thus the right hand size of \eqref{eq:telescoping-sum}
can be further bounded by
\[
|U^n(x)-u^1(x, n\eta)|\le C(f,\varphi,R)\eta^2
\]
for some positive constant $C(f,\varphi,R)$ independent of $n$. This completes
the proof.
\end{proof}

The key contribution of Theorem~\ref{thm:sgdasym} is the
extension of the range of applicability of diffusion
approximation \eqref{eq:diffusion-approx-prev} from finite
time interval $\left[ 0,T \right]$ to infinite time. A
direct consequence is the following description of the ``weak
expansion'' of the stationary distribution of the dynamics
\eqref{eq:sgd-dynamics}.

\begin{corollary}\label{cor:behaviorofUn}
Under the same conditions as in Theorem~\ref{thm:sgdasym},
we have for all $n\gtrsim \frac{1}{\eta}\log (1/\eta)$ that
\[
\sup_{x\in B(0,R)}\left|\mathbb{E}_x\varphi(X_n)-\varphi(0)\right|=\sup_{x\in B(0,R)}\left|U^n(x)-\varphi(0)\right|\le C(\varphi, f,R)\eta
\]
for some positive constant $C \left( \varphi, f, R \right)$. Moreover, the probability measure in Lemma \ref{lmm:w2conv} satisfies
\[
\left|\int_{\mathbb{R}^d} \varphi\,\mathrm{d}\pi-\varphi(0)-\eta \varphi_1 \right|\le C\eta^2,
\]
where $\varphi_1=\lim_{t\to\infty}u_1(x, t)=\int_0^{\infty}\cL_2u_0(0,s)\,ds$ is independent of $x$.
\end{corollary}
The conclusion follows immediately from noting that, for
$n\gtrsim \eta^{-1}\log(\eta^{-1})$,
\[
|u^1(x, n\eta)-\varphi(0)|=|u_0(x,n\eta)+\eta u_1(x, n\eta)-\varphi(0)|\le C(f,\varphi,R)\eta.
\]
In particular, if we choose $\varphi=f$,
Corollary~\ref{cor:behaviorofUn} tells us that SGD descends
the value of a strongly convex objective function to an
$\mathcal{O}\left( \eta \right)$ neighborhood of the global
minimum in only $\mathcal{O} \left( \eta^{-1}\log (\eta^{-1})
\right)$ time. Measured in the time scale of diffusion
approximation, where $t=n\eta$ in $u^1 \left( x,n\eta
\right)$, this is equivalent to say that the SGD dynamics
reduces the objective value to $\mathcal{O}\left( \eta
\right)$ away from the global minimum within time
$n\eta=\mathcal{O}\left( \log (1/\eta) \right)$, which is
exponentially fast, as well known.

At last, we remark that if $X_0$ starts with a measure $\mu_0$ instead of $X_0=x$, then
$\int (u_0+\eta u_1)(t, x)\mu_0(dx)$ will approximate $\mathbb{E}\varphi(X_n)$ uniformly in time. We may further rewrite the quantity as
\begin{gather}
\int_{\mathbb{R}^d} (u_0(x,t)+\eta u_1(x,t))\mu_0(dx)=\int_{\mathbb{R}^d} \varphi(x) (\nu_0(dx)+\eta\nu_1(dx)),
\end{gather}
with $\nu_0, \nu_1$ respectively satisfying (see Appendix \ref{eq:formaleq} for a formal derivation):
\begin{gather}\label{eq:nu0}
\partial_t \nu_0-\nabla\cdot(\nabla f \nu_0)=0,~~\nu_0(0)=\mu_0,
\end{gather}
and
\begin{gather}\label{eq:nu1}
\partial_t\nu_1-\nabla\cdot(\nabla f\nu_1)=\frac{1}{4}\nabla\cdot(\nabla\|\nabla f\|^2 \nu_0)+\frac{1}{2}\partial_{ij}(\Sigma_{ij}\nu_0),~~\nu_1(0)=0.
\end{gather}
Theorem \ref{thm:sgdasym} then implies that $\nu^1:=\nu_0+\eta\nu_1$ or $\frac{\nu_0+\eta\nu_1}{M(\nu_0+\eta\nu_1)}$ approximates
the distribution of $X_n$ with second weak order, where $M(\nu)$ means the total mass of $\nu$: 
\[
M(\nu):=\int_{\mathbb{R}^d}d\nu.
\]

\begin{remark}\label{rmk:higher}
The weak order of approximation $\mathcal{O}\left( \eta^2
\right)$ in Theorem~\ref{thm:sgdasym} is optimal in the
sense that no higher order approximation error can be
achieved by choosing $N>1$ in \eqref{eq:truncatedsum},
although the formal truncated series $u^N$ may better
approximate the Kolmogorov equation
\eqref{eq:backwardKol}. This is because the diffusion approximation
itself is only a weak second order approximation for SGD
\cite[Theorem 2.2]{feng2017}. Higher order
approximation for the SGD dynamics requires higher
derivatives of $u$ in the PDE \eqref{eq:backwardKol}, but
it no longer describes a diffusion process (solutions of It\^o
equations).
\end{remark}

\section{Numerical Experiments}
\label{sec:numer-exper}

In this section we demonstrate the approximation power of
the truncated formal series \eqref{eq:truncatedsum} with
numerical experiments for some one-dimensional ($d=1$)
examples. We consider SGD schemes
\begin{equation}
\label{eq:sgd-example-one-dim}
  f \left( x;\xi \right)=f \left( x \right)+\frac{\xi}{2}x,\quad x\in \mathbb{R}
\end{equation}
where $f:\mathbb{R}\rightarrow\mathbb{R}$ is locally
strongly convex near one of its local minima, and $\xi$ is a
Rademacher random variable that assigns equal probability
$1/2$ to both $-1$ and $+1$. Following the definitions in
\eqref{eq:differential-ops}, we have explicitly
\begin{align}
   \cL_1 =-f' \left( x \right)\partial_x,\qquad\cL_2
  =-\frac{1}{2}f' \left( x \right)f'' \left( x \right)\partial_x+\frac{1}{8}\partial_x^2.
\end{align}
The first two terms in the formal series expansion
\eqref{eq:asyexp} can be determined by solving the two first
order PDEs sequentially: First solve
\begin{equation}
  \begin{aligned}
    &\partial_t u_0+f' \left( x \right)\partial_x u_0 \left( x,t \right)=0\\
    & u_0(x,0)=\varphi(x)
  \end{aligned}
\end{equation}
to get
\begin{equation}
\label{eq:u0ana}
  u_0 \left( x,t \right)=\varphi \left( x_0 \left( x,t \right)\right),
\end{equation}
where $x_0 \left( x,t \right)$ is the intercept of the
characteristic line passing through $\left( x,t
\right)\in\mathbb{R}\times\mathbb{R}_+$. We then use
\eqref{eq:u0ana} to solve
\begin{equation}
  \begin{aligned}
    &\partial_t u_1+f' \left( x \right)\partial_x u_1 \left(
      x,t \right)=-\frac{1}{2}f' \left( x \right)f'' \left( x \right)\partial_xu_0+\frac{1}{8}\partial_x^2 u_0\\
    & u_1(x,0)=0
  \end{aligned}
\end{equation}
which gives
\begin{equation}
\begin{aligned}
  u_1&\left( x,t \right)=-\frac{1}{2}f' \left( x_0 \left( x,t \right) \right)\varphi' \left( x_0\left( x,t \right) \right)\log \frac{f' \left( x \right)}{f' \left( x_0\left( x,t \right) \right)}+\frac{1}{8}f' \left( x_0 \left( x,t \right)\right)f'' \left( x_0 \left( x,t \right) \right)\varphi' \left( x_0 \left( x,t \right)\right)\int_{x_0 \left( x,t \right)}^x \frac{\mathrm{d}\xi}{\left[ f' \left( \xi \right) \right]^3}\\
   &-\frac{1}{16}f' \left( x_0 \left( x,t \right)\right)\varphi' \left( x_0 \left( x,t \right)\right)\left\{ \frac{1}{\left[ f' \left( x_0 \left( x,t \right) \right) \right]^2}-\frac{1}{\left[ f' \left( x \right) \right]^2}\right\}+\frac{1}{8}\left[ f' \left( x_0 \left( x,t \right) \right) \right]^2\varphi'' \left( x_0 \left( x,t \right)\right)\int_{x_0 \left( x,t \right)}^x \frac{\mathrm{d}\xi}{\left[ f' \left( \xi \right) \right]^3}.
\end{aligned}
\end{equation}
Details of this computation can be
found in Appendix~\ref{sec:examples-full}.

\begin{example}
\label{exm:simple}
  We consider a simple example
\begin{equation}
\label{eq:sgd-example-1}
  f \left( x \right)=\frac{1}{2}x^2-\frac{1}{2}x.
\end{equation}
The stochastic gradient updates are
\begin{equation*}
  X_{n+1}=X_n-\eta\nabla f \left( X_n; \xi_n \right)=\left(
    1-\eta \right)X_n-\frac{\eta}{2} \left(1-\xi_n\right)
\end{equation*}
where $\left\{ \xi_n \right\}_{n\geq 0}$ are
i.i.d. standard Rademacher random
variables. The limiting distribution of this Markov chain is identical
to that of $X_{\infty}=\eta \sum_{j=0}^{\infty}\theta_j \left( 1-\eta
  \right)^j$ where the $\theta_j$'s are
  i.i.d. $\mathrm{Bernoulli}\left( 1/2 \right)$ random
  variables. The infinite series converges whenever
  $\eta\in \left( 0,1 \right)$, but the stationary distribution is drastically
  different for
different values of $\eta$ \cite[\S2.5]{DF1999}:
If $\eta=1/2$, $X_{\infty}$ is uniformly distributed on
$\left[ 0,1 \right]$; if $1/2<\eta<1$, the distribution of
$X_{\infty}$ is singular (supported on a set of Lebesgue
measure $0$); if $0<\eta<1/2$, for some values of $\eta$ the
stationary distribution is singular, but it has also been
established that for almost all $\eta\in \left( 0,1/2
\right)$ the stationary distribution is absolutely
continuous. We are most interested in the regime $\eta\in
\left( 0,1/2 \right)$ where $\eta$ is small.

We choose several different test functions $\varphi$ to
verify the order of the weak approximation error between
$U^n \left( x \right) = \mathbb{E}_x\left[ \varphi \left( X_n \right) \right]$ and
$u^1=u_0+\eta u_1$ established in
Theorem~\ref{thm:sgdasym}. The results are summarized in
Figure~\ref{fig:loglog} and Figure~\ref{fig:approx}.

\begin{figure}[h]
  \centering
  \includegraphics[width=1.0\textwidth]{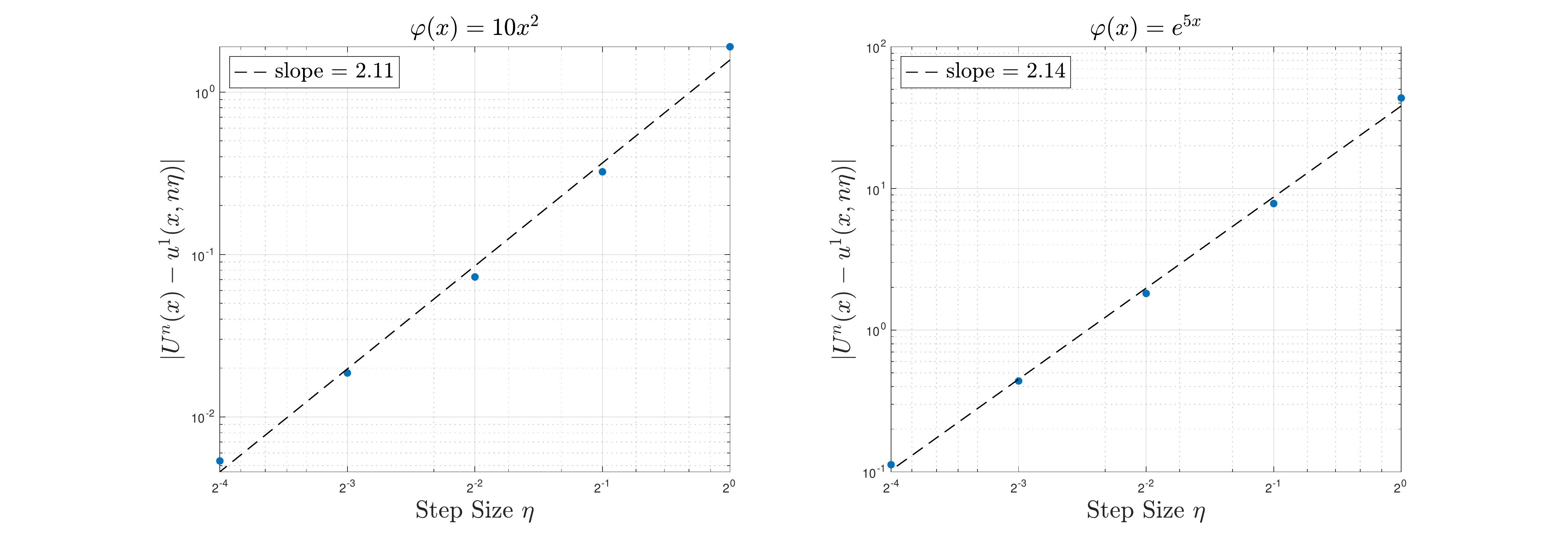}
  \caption{Log-log plots numerically verifying the
    weak second order diffusion approximation established in
    Theorem~\ref{thm:sgdasym}, using example
    \eqref{eq:sgd-example-1} and two different test functions
    $\varphi$. For each $\varphi$, we fix $x=1$ and
    $n\eta=5$, then let $\eta$
    vary in $\{ 2^{-4},2^{-3},2^{-2},2^{-1},2^0 \}$. We use
    a Monte--Carlo simulation to evaluate $U^n \left(
      x\right)=\mathbb{E}_x \left[ \varphi\left(X_n\right)
    \right]$, by averaging $\varphi \left( X_n \right)$ over
    $10^8$ independent trajectories starting from
    $X_0=x$. The slopes of the fitting lines are close to
    $2$, which justify the second order approximation established
   in Theorem~\ref{thm:sgdasym}.}
  \label{fig:loglog}
\end{figure}

\begin{figure}[h]
  \centering
  \includegraphics[width=1.0\textwidth]{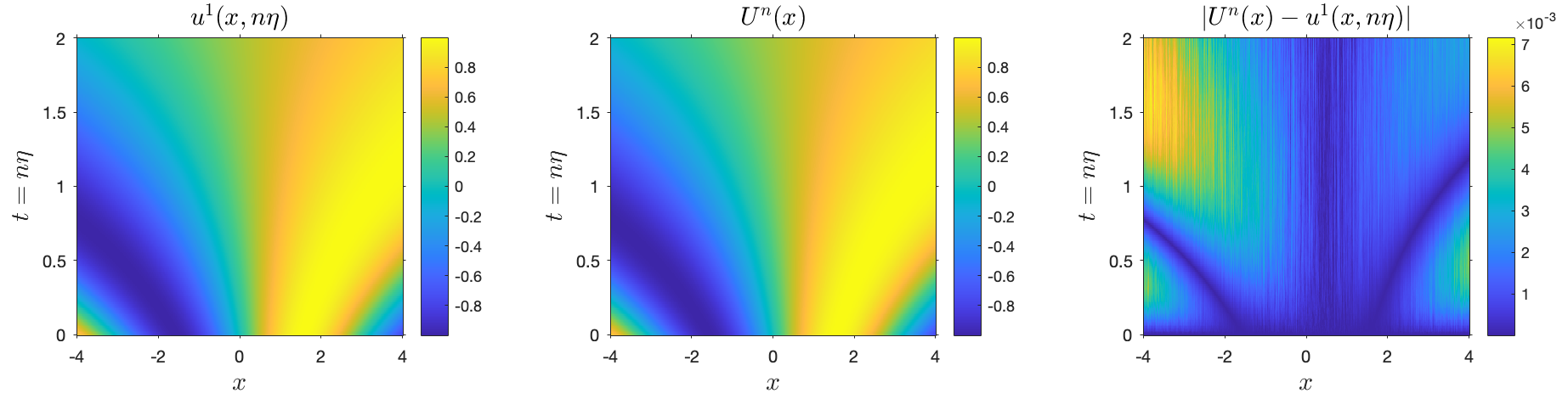}
  \caption{Visual comparison of $u^1 \left( x,n\eta
    \right)$ and $U^n \left( x \right)$ for $\varphi \left(
      x \right)=\sin \left( x \right)$ over $\left( x,t
    \right)\in \left[ -4,4 \right]\times\left[
      0,2 \right]$, with $\eta=0.01$. Each $U^n \left(
      x\right)$ is evaluated over $10^4$ independent
    trajectories generated from the gradient dynamics associated with \eqref{eq:sgd-example-1}.}
  \label{fig:approx}
\end{figure}
\end{example}

\begin{example}
We now consider a more complicated example in which the gradient $\nabla f$
is nonlinear. Set
\begin{equation}
\label{eq:sgd-example-2}
  f \left( x \right)=\frac{1}{2}x^2+0.1x^3
\end{equation}
and the stochastic gradient updates can be written as
\begin{equation*}
  X_{n+1}=X_n-\eta\nabla f \left( X_n; \xi_n
  \right)=\left( 1-\eta \right)X_n-0.3\eta X_n^2-\frac{\eta}{2}\xi
\end{equation*}
where $\left\{ \xi_n \right\}_{n\geq 0}$ are
i.i.d. standard Rademacher random
variables. We choose the same test functions $\varphi$ as in
Example~\ref{exm:simple}. The results are summarized in Figure~\ref{fig:loglog2}.
  \begin{figure}[h]
  \centering
  \includegraphics[width=1.0\textwidth]{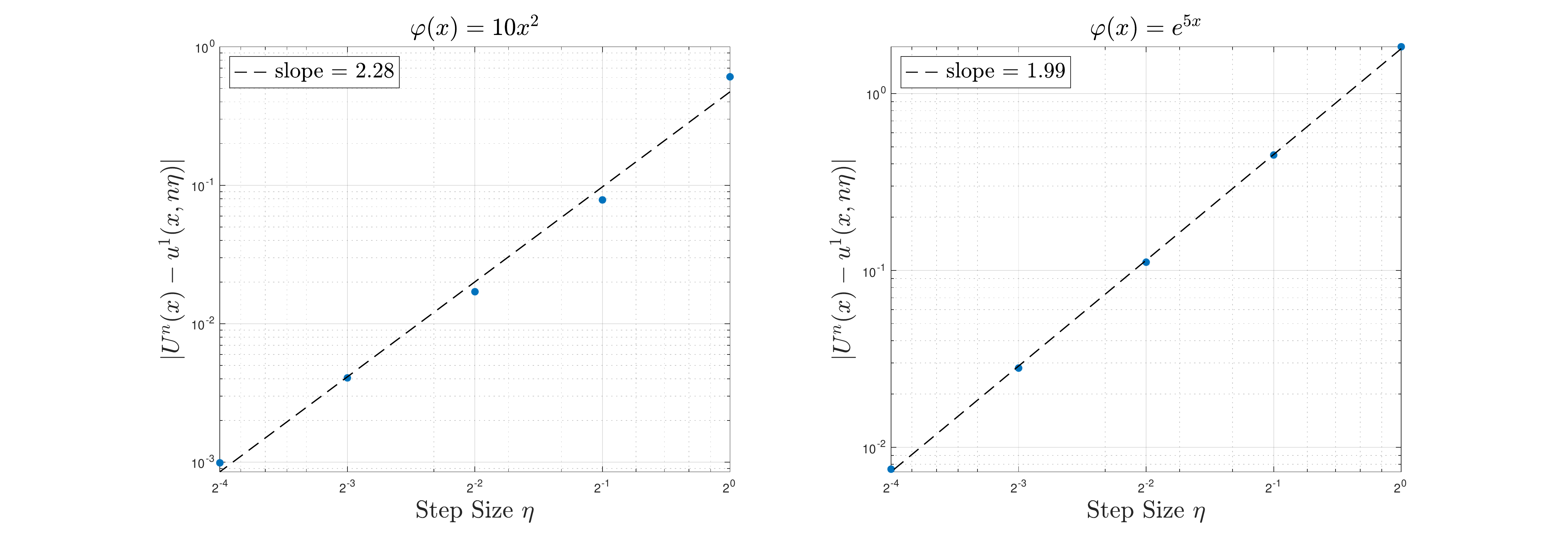}
  \caption{Log-log plots numerically verifying the
    weak second order diffusion approximation established in
    Theorem~\ref{thm:sgdasym}, using example
    \eqref{eq:sgd-example-2} and two different test functions
    $\varphi$. For each $\varphi$, we fix $x=1$ and
    $n\eta=5$, then let $\eta$
    vary in $\{ 2^{-4},2^{-3},2^{-2},2^{-1},2^0 \}$. We use
    a Monte--Carlo simulation to evaluate $U^n \left(
      x\right)=\mathbb{E}_x \left[ \varphi\left(X_n\right)
    \right]$, by averaging $\varphi \left( X_n \right)$ over
    $10^8$ independent trajectories starting from
    $X_0=x$. The slopes of the fitting lines are close to
    $2$, which justify the second order approximation established
   in Theorem~\ref{thm:sgdasym}.}
  \label{fig:loglog2}
\end{figure}
\end{example}

\section{Conclusion}
\label{sec:conclusion}

In this paper, we establish uniform-in-time weak error bounds for diffusion approximation of SGD algorithms, under the local strong convexity assumption for the objective functions. To this end, we adapted the idea of backward error analysis in numerical SDEs, and used a truncated formal series expansion with respect to the constant step size for the backward Kolmogorov equation associated with the modified SDE---instead of the solution itself---to approximate the SGD iterates for arbitrarily long time. This enables us to draw quantitative conclusions for the weak asymptotic behavior of the SGD iterates from
estimates of the coefficient functions of the truncated formal expansion, which is the first result of this type for diffusion-approximation-based SGD analysis. We believe the tools developed in this paper have great potential in
generalizing the range of applicability of diffusion approximation to many other stochastic optimization algorithms in data science, such as SGD with non-constant step size and momentum-based acceleration techniques.

\section*{Acknowledgement}
The work of J.-G. Liu was partially supported by KI-Net NSF RNMS11-07444 and NSF DMS-1812573.
The work of L. Li was partially sponsored by NSFC 11901389, Shanghai Sailing Program 19YF1421300 and NSFC 11971314. 
The work of T. Gao was partially supported by NSF DMS-1854831.

\appendix

\section{Proofs of Technical Lemmas in
  Section~\ref{sec:form-series-expans}}
\label{sec:technical-lemmas}

\begin{proof}[Proof of Lemma~\ref{lmm:uniformbound}]
  By \eqref{eq:sgd-dynamics}, we have
\begin{align}\label{e:ltmp1}
\left\|X_{n+1}\right\|^2&=\left\|X_n\right\|^2-2\eta\left(\nabla f(X_n; \xi)-\nabla f(0; \xi)\right)\cdot X_n+\eta^2\left\|\nabla f(X_n; \xi)\right\|^2-2\eta \nabla f(0; \xi)\cdot X_n\\
\nonumber
&\leq \left\|X_n\right\|^2-2\gamma \eta \left\|X_n\right\|^2+\eta^2(b+\gamma|X_n|)^2-2\eta \nabla f(0; \xi)\cdot X_n\,,
\end{align}
where we applied the strong convexity of $f(\cdot; \xi)$ in the last inequality and the fact that
\[
|\nabla f(X_n; \xi_n)|\le |\nabla f(0;\xi_n)|+\gamma |X_n|\le b+\gamma |X_n|.
\]. 
When $\left\|X_n\right\| \leq \frac{R}{2}$, \eqref{e:ltmp1} can be further controlled by
\begin{align*}
\left\|X_{n+1}\right\|^2 &\leq (1-2\gamma \eta)\frac{R^2}{4}+\eta^2b^2+\eta^2\gamma b R
+\eta^2\frac{\gamma^2 R^2}{4}+\eta b R.
\end{align*}
Noting that $-\frac{1}{2}\gamma \eta R^2+ \eta^2\frac{\gamma^2 R^2}{4}\le -\frac{3}{8}\gamma \eta R^2$, we find
\begin{align*}
\left\|X_{n+1}\right\|^2 & \leq \frac{R^2}{4}+\gamma \eta R(-\frac{3R}{8}+\eta b)
+\eta b(\eta b+R)\\
&\le \frac{R^2}{4}+\frac{3R}{8}*\frac{11R}{8} < R^2.
\end{align*}

When $\frac{R}{2}\leq |X_n|\leq R$, we have
\begin{align*}
\left\|X_{n+1}\right\|^2 &\leq \left\|X_n\right\|^2+(-2\gamma
                           \eta +\eta^2\gamma^2) \cdot \frac{R^2}{4}+2\gamma b R\eta^2+\eta^2b^2+2\eta  bR\\
                     &\le \left\|X_n\right\|^2+(2b-\frac{3\gamma  R^2}{8})\eta
                     +\eta^2(2\gamma b R+b^2) \leq |X_n|^2.
\end{align*}
Thus the conclusion follows.
\end{proof}

\begin{proof}[Proof of Lemma~\ref{lmm:w2conv}]
  Consider two copies of the chain
\begin{gather}
\begin{split}
Y_{n+1}=Y_n-\eta \nabla f(Y_n; \xi_n ),\quad Z_{n+1}=Z_n-\eta \nabla f(Z_n; \xi_n).
\end{split}
\end{gather}
The two chains are coupled through the random variable
$\xi_n$. This means that they pick the same function to
compute the gradient at every iteration $n$. Meanwhile, each chain
has the same asymptotic distributional behavior as the SGD. We then have
\begin{gather*}
\mathbb{E}\left\|Y_{n+1}-Z_{n+1}\right\|^2
=\mathbb{E}\left\|Y_n-Z_n\right\|^2-2\eta \mathbb{E}\left[(Y_n-Z_n)\cdot(\nabla f(Y_n, \xi_n)-\nabla f(Z_n, \xi_n))\right]
+\eta^2 \mathbb{E}\left\|\nabla f(Y_n, \xi_n)-\nabla f(Z_n, \xi_n)\right\|^2.
\end{gather*}
For the second term, we use conditional expectation to
deduce that
\begin{multline*}
\mathbb{E}\left[(Y_n-Z_n)\cdot(\nabla f(Y_n, \xi_n)-\nabla f(Z_n, \xi_n))\right]\\
=\mathbb{E}\left[ (Y_n-Z_n)\cdot \mathbb{E}\left[\nabla f(Y_n, \xi_n)-\nabla f(Z_n, \xi_n) | Y_m, Z_m, m\le n\right]\right]\\
=\mathbb{E}\left[(Y_n-Z_n)\cdot(\nabla f(Y_n)-\nabla f(Z_n))\right]\ge \gamma\mathbb{E}\left\|Y_n-Z_n\right\|^2.
\end{multline*}
The last term is upper bounded by
\[
\eta^2\mathbb{E}\left\|\nabla f(Y_n, \xi_n)-\nabla f(Z_n, \xi_n)\right\|^2 \le \eta^2L^2 \mathbb{E}|Y_n-Z_n|^2.
\]
Therefore, it follows that
\[
\mathbb{E}\left\|Y_{n+1}-Z_{n+1}\right\|^2\le (1-2\gamma \eta+\eta^2 L^2)\mathbb{E}\left\|Y_n-Z_n\right\|^2.
\]
Now, if $\eta<2\gamma/L^2$, then $0<1-2\gamma\eta+\eta^2 L^2<1$. We claim that under this choice of $\eta$, the law of $X_n$ is a Cauchy sequence under the $W_2$ norm. In fact, for any $\epsilon>0$, we can pick $m>0$ such that $(2R)^{2}(1-2\gamma\eta+\eta^2L^2)^m<\epsilon^2/4$.  For $n\ge m$, we pick $Y_0$ to have the same distribution as $X_0$ and $Z_0$ to have the same distribution as $X_{n-m}$. Then, $Y_m$ has the same distribution as $X_m$ while $Z_m$ has the same distribution as $X_n$. Moreover,
\begin{gather}\label{eq:YZcouple}
\mathbb{E}\left\|Y_{m}-Z_{m}\right\|^2\le (1-2\gamma\eta+\eta^2 L^2)^m\mathbb{E}\left\|Y_{0}-Z_{0}\right\|^2< \epsilon^2/4.
\end{gather}
It follows that
\[
\left(\mathbb{E}\left\|Y_m-Z_m\right\|^2\right)^{1/2}<\epsilon/2.
\]

We recall that the Wasserstein-$2$ distance is given by
\begin{gather}\label{eq:W2}
W_2(\mu, \nu)=\left(\inf_{\gamma \in \Pi(\mu,\nu)}\int_{\mathbb{R}^d\times\mathbb{R}^d}|x-y|^2 d\gamma\right)^{1/2},
\end{gather}
where $\Pi(\mu,\nu)$ means the set of all the joint distributions $\gamma$ whose marginal distributions are $\mu$ and $\nu$ respectively.
Since the joint distribution of $(Y_m, Z_m)$ is in $\Pi(\mu_n, \mu_m)$, one finds $W_2(\mu_n, \mu_m)<\epsilon/2$. This means that
$\mu_n$ is a Cauchy sequence, and it holds for some probability distribution $\pi$ that
\[
\lim_{n\to\infty}W_2(\mu_n, \pi)=0.
\]
 Finally, we obtain
from \eqref{eq:YZcouple} that
\[
W_2(\mu_n, \mu_m)\le (1-2\gamma\eta+\eta^2 L^2)^{m/2}\sqrt{\mathbb{E}\left\|Y_{0}-Z_{0}\right\|^2}
\le C(1-2\gamma\eta+\eta^2 L)^{m/2},
\]
where $C$ is independent of $m,n$ (the second moment of
$X_{n-m}$ is uniformly bounded). The conclusion follows from
taking the limit $n\to\infty$.
\end{proof}

\section{Proof of the Exponential Decay Estimates}
\label{sec:proof-main-technical-thm}

\begin{proof}[Proof of Theorem~\ref{p:main}]
The genesis of the exponential decay rates of the $u_{\ell}$'s can be traced
back to the following simple yet important observation: Suppose $y(t)$ satisfies
\begin{align}\label{e:trans}
\dot{y}=-\nabla f(y)
\end{align}
with $y(0)=x$, then $\left\|y(t)\right\|$ is a non-increasing function and
\begin{align}\label{e:expdecay}
\left\|y(t)\right\|\le \left\|y(0)\right\| e^{-\gamma t}.
\end{align}

We now begin our proof. First by the method of characteristics \cite[Theorem 5.34]{Vil03}, one notices that $u_0$ satisfies
\begin{align} \label{e:u0}
&\partial_tu_0+\nabla f(x)\cdot\nabla u_0=0 \,,\\
\label{e:u0initial}
&u_0(x,0)=\varphi(x)\,.
\end{align}
Let $y$ be the function in~\eqref{e:trans} with $y(0)=x\in B(0,R)$. And for any given $T>0$, $t\in [0,T]$, define $z(t):=y(T-t)$. Then it follows that
\begin{align*}
u_0(z(t), t)=\varphi(z(0))\,,~\forall t\in [0,T]\,.
\end{align*}
Consequently, we have $u_0(x, t)=\varphi(y(t))\,,~ \forall t>0\,.$
Hence,
\[
|u_0(x,t)-\varphi(0)|\le \|\nabla\varphi\|_{L^{\infty}(B(0,R))}|y(t)|\le R\|\varphi\|_{C^1(B(0,R))}e^{-\gamma t}.
\]
For the estimate of derivatives, we use induction. When $k=1$, following from equations~\eqref{e:u0} and~\eqref{e:u0initial}, we have
\begin{align*}
\partial_t \left\|\nabla u_0\right\|^2= -2\nabla u_0 \cdot \nabla^2 f \cdot\nabla u_0-\nabla f \cdot \nabla \left\|\nabla u_0\right\|^2\quad\textrm{and}\quad\left\|\nabla u_0(x,0)\right\|^2=\left\|\nabla \phi(x)\right\|^2.
\end{align*}
Since $f$ is strongly convex,
\begin{align}\label{e:deri}
\partial_t \left\|\nabla u_0\right\|^2 \leq -2\gamma\left\|\nabla u_0\right\|^2-\nabla f \cdot \nabla \left\|\nabla u_0\right\|^2 \,.
\end{align}
Recall $y(t)$ which was defined in equation~\eqref{e:trans} and $z(t)=y(T-t)$.  By chain rule, equation~\eqref{e:deri} yields that $\frac{d}{dt}\|\nabla u_0 (z(t),t))\|^2\leq-2\gamma \|\nabla u_0(z(t),t)\|^2 $, which by Gronwall's inequality further yields 
\begin{align*}
\left\|\nabla u_0(z(t),t)\right\|\leq e^{-\gamma t}\|\nabla u_0(z(0),0)\| \leq e^ {-\gamma t}\left\|\nabla \varphi(z(0))\right\| \,, ~\forall t \in [0,T]\,.
\end{align*}
This then yields
\begin{align*}
\left\|\nabla u_0(x,t)\right\|&\leq e^{-\gamma t}\left\|\nabla \varphi (y(t))\right\|\leq \|\nabla \varphi\|_{L^\infty(B(0,R))} e^{-\gamma t}\\
&\leq \|\varphi\|_{C^1(B(0,R))}e^{-\gamma t}\,,~\forall t>0\,, x\in B(0,R)\,.
\end{align*}
Hence inequality~\eqref{e:u0deri} is verified for $k=1$.
By induction, we assume for any $k\leq m$, inequality~\eqref{e:u0deri} holds. Next we study the case for $k=m+1$. For $J\in I_{m+1}$, we differentiate equation~\eqref{e:u0} by $\partial^{J}$ and get $\partial_t \partial^J u_0 +\partial ^J(\nabla f\cdot \nabla u_0)=0$. Then multiplying both sides by $\partial^J u_0$ and summing over all  $J\in I_{m+1}$ gives
\begin{align*}
\partial_t v =-2 \sum_{J\in I_{m+1}}\sum_{i=1}^d \partial^J u_0\partial^J(\partial_i f\partial_i u_0)\,,
\end{align*}
where $v=\sum_{J\in I_{m+1}}(\partial^J u_0)^2$\,.
 We note that the right hand side can be splitted into the sum of three terms according to the general Leibniz rule in calculus. And then the above equation becomes
 \begin{align}\label{e:final}
\partial_t v \leq -2 (n+1)\gamma v-\nabla f \cdot \nabla v
-2 \sum_{i=1}^d\sum_{J\in I_{m+1}}\partial^J u_0
\sum_{J_0\leq J,|J_0|\geq 2} \partial^{J_0} \partial_i f \partial^{J-J_0}\partial_i u_0\,.
\end{align}
 Here is a brief explanation of the above inequality~\eqref{e:final}. For  $k\in \{1,\cdots, {m+1}\},~ j_k \in \{1,\cdots,d\}$, putting the first order derivative $\partial_{j_k}$ on $\partial_i f$ and  $\partial^{J- \{j_k\}}$ on $\partial_iu_0$, we would obtain
 \begin{align*}
 -2\sum_{J\in I_{m+1}}\sum_{k=1}^{m+1}\,\sum_{i, j_k=1}^d\partial^J u_0 \partial_{j_k}\partial_i  f \partial^{J-\{j_k\}}\partial_i u_0\,,
 \end{align*}
which is a quadratic form associated with the Hessian matrix $\nabla^2 f$. This also explains why we do not use the traditional  definition of multi-index in our paper (the question related to Remark~\ref{r:multin}). By the strong convexity of $f$, the above term is bounded above by
\begin{align*}
-2\sum_{J\in I_{m+1}}\sum_{k=1}^{m+1}\,\sum_{i,j_k=1}^{d} \gamma ( \partial^{J-\{j_k\}} \partial_i u_0)^2\,,
\end{align*}
which can be further bounded above by $-2(m+1)\gamma v$\,. To put all the $J^{th}$ derivative on $\partial_i u_0$ yields to the second term $-\nabla f\cdot \nabla v$\,.
 For the third term, we only need to consider the rest terms due to the Leibniz rule. Hence the validity of~\eqref{e:final} has been proved.

 For the last term in~\eqref{e:final} , we use Young's inequality
and the induction assumption,  then derive that
\begin{align*}
\partial_t v \leq -2m \gamma v - \nabla f \cdot \nabla v +
  P\Big(\|\varphi\|_{C^m(B(0,R))},
  \|f\|_{C^{m+2}(B(0,R))}\Big)e^{-2\gamma t}\,.
\end{align*}
We also note that $|v(x,0)|\leq \|\varphi \|_{C^{m+1}(B(0,R))} $, for $x\in B(0,R)$. Hence we get
\begin{align*}
v(z(t),t)\leq P\Big(\|\varphi\|_{C^{m+1}(B(0,R))}, \|f\|_{C^{m+2}(B(0,R))}\Big) e^{-2\gamma t}\,, ~\forall t\in [0,T]\,.
\end{align*}
This then gives
\begin{align*}
v(x,t)\leq P\Big(\|\varphi\|_{C^{m+1}(B(0,R))}, \|f\|_{C^{m+2}(B(0,R))}\Big) e^{-2\gamma t}\,, ~\forall t>0\,.
\end{align*}
Hence result~\eqref{e:u0deri} is proved.

Now we start to study $u_n$. The equation which $u_n$ satisfies is the following
\begin{align*}
&\partial_tu_n+\nabla f\cdot\nabla u_n=\cL_2u_{n-1}\,,\\
& u_n(x,0)=0\,.
\end{align*}
Based on this, we could write down a formula for $u_n$,
\begin{gather}\label{eq:unformula}
u_n(x, t)=\int_0^t \cL_2u_{n-1}(y(s),t- s)\,ds\,.
\end{gather}
Here we recall that $y$ satisfies equations \eqref{e:trans} with $y(0)=x\in B(0,R)$ and thus \eqref{e:expdecay}.

Consider $n=1$. For convenience,  we denote
\[
g(x, t)=\cL_2u_0(x,t).
\]
Intuitively, the limiting behavior of $u_1(x, t)$ is determined by $g(0, t)$. We now verify this.

Recall the definition of the operator $\cL_2$~\eqref{eq:differential-ops}, we have
\begin{align}\label{e:L2exp}
\nonumber
\sup_{x\in B(0,R)}(|g(x,t)|+|\nabla g(x, t)|)
&\leq C\Big(\|f\|_{C^3(B(0,R))}+ \|\Sigma\|_{C^1 (B(0,R))}\Big)\|u_0\|_{C^3(B(0,R))}\\
&\leq P\Big(\|\varphi\|_{C^3(B(0,R))},\|f\|_{C^4(B(0,R))},\|\Sigma\|_{C^1 (B(0,R))} \Big)e^{-\gamma t}
 \,,
\end{align}
where the last inequality followed from~ \eqref{e:u0deri}.
It follows that $\sup_{x\in B(0, R)}|u_1(x,t)|$ is uniformly bounded in $t$. Moreover,  we further split $u_1$ as
\begin{gather}
u_1(x, t)=\int_0^t g(0, t-s)\,ds+\int_0^t(g(y(s), t-s)- g(0, t-s))\,ds.
\end{gather}

The second term is controlled by directly by \eqref{e:L2exp} as
\begin{gather*}
\begin{split}
\left|\int_0^t(g(y(s), t-s)- g(0, t-s))\,ds \right| &\le \int_0^{t} \|\nabla g(\cdot, t-s)\|_{L^{\infty}(B(0, R))} 
|y(s)|\,ds \\
& \le P\Big(\|\varphi\|_{C^3(B(0,R))},\|f\|_{C^4(B(0,R))},\|\Sigma\|_{C^1 (B(0,R))} \Big) \int_0^t e^{-\gamma(t-s)}e^{-\gamma s}\,ds\\
& = P\Big(\|\varphi\|_{C^3(B(0,R))},\|f\|_{C^4(B(0,R))},\|\Sigma\|_{C^1 (B(0,R))} \Big) te^{-\gamma t} \\
&\leq  P\Big(\|\varphi\|_{C^3(B(0,R))},\|f\|_{C^4(B(0,R))},\|\Sigma\|_{C^1 (B(0,R))} \Big)  e^{-\gamma' t}\,,
\end{split}
\end{gather*}
where the last inequality followed from $t e^{-\gamma t}\le C(\gamma')e^{-\gamma' t}$ for any $\gamma'<\gamma$.

Regarding the first term in \eqref{e:L2exp}, we know that it converges to
\[
\varphi_1 =\int_0^\infty g(0,s)\,ds,
\]
 with the exponential rate. Hence, overall, we have
\begin{gather}
|u_1(x,t)-\varphi_1| \le P\Big(
\|\varphi\|_{C^3(B(0,R))}, \|f\|_{C^4(B(0,R))},
\|\Sigma\|_{C^1 (B(0,R))}\Big ) e^{-\gamma' t}\,.
\end{gather}

For the derivatives of $u_1$, we notice that
\begin{align}\label{e:u1deri}
\partial_t \left\|\nabla u_1\right\|^2= -2 \nabla u_1\cdot
  \nabla^2 f \cdot \nabla u_1-\nabla f\cdot \nabla \left\|\nabla u_1\right\|^2 + 2\sum_{j=1}^d\partial_j u_1\partial_j(\cL_2 u_0)\,.
\end{align}
Also we notice that
\begin{align*}
\sup_{x\in B(0,R)}\left|\partial_j (\cL_2
  u_0(x,t))\right|\leq
  C\Big(\|f\|_{C^3(B(0,R))}+\|\Sigma\|_{C^1(B(0,R))}\Big)\|u_0\|_{C^3(B(0,R))}\,.
\end{align*}
We use this in \eqref{e:u1deri}, and for the first term we use strong convexity of $f$ as well,  then get
\begin{align*}
\partial_t  \left\|\nabla u_1\right\|^2 &\leq -2 \gamma \left\|\nabla u_1\right\|^2 -\nabla f\cdot \nabla \left\|\nabla u_1\right\|^2\\
&~+ 2\sum_{j=1}^d |\partial_j u_1|~ P\Big(\|\varphi\|_{C^3(B(0,R))},\|f\|_{C^4(B(0,R))},\|\Sigma\|_{C^1(B(0,R))} \Big) e^{-\gamma t}\,.
\end{align*}
We then apply Young's inequality to further get for any $\gamma''< \gamma$,  there exists a polynomial $P$ in $\|\varphi\|_{C^3(B(0,R))}$, \\$\|f\|_{C^4(B(0,R))}$ and $\|\Sigma\|_{C^1(B(0,R))} $ such that 
\begin{align*}
\partial_t  \left\|\nabla u_1\right\|^2 &\leq - 2\gamma'' \left\|\nabla u_1\right\|^2 -\nabla f\cdot \nabla \left\|\nabla u_1\right\|^2\\
&~+P\Big(\|\varphi\|_{C^3(B(0,R))},\|f\|_{C^4(B(0,R))},\|\Sigma\|_{C^1(B(0,R))} \Big) e^{-2\gamma t}\,.
\end{align*}
Hence it holds that
\begin{align}
\sup_{x\in B(0,R)}\left\|\nabla u_1(x,t)\right\|\leq P\Big(\|\varphi\|_{C^3(B(0,R))},\|f\|_{C^4},\|\Sigma\|_{C^1(B(0,R))} \Big) e^{-\gamma'' t}\,.
\end{align}
For the higher derivatives of $u_1$, the analysis goes similarly as that of $u_0$. We also use induction here. Assume for any $k\leq m$, \eqref{e:underi} holds. For $k=m+1$, we denote $w=\sum_{J\in I_{m+1}}(\partial^Ju_1)^2$ and get
\begin{align*}
\partial_t w &\leq -2 (m+1)\gamma w-\nabla f \cdot \nabla w
-2 \sum_{i=1}^d\sum_{J\in I_{m+1}}\partial^J u_1
\sum_{J_0\leq J,|J_0|\geq 2} \partial^{J_0} \partial_i f \partial^{J-J_0}\partial_i u_1
+2\sum_{J\in I_{m+1}}(\partial^J u_1)\partial^J\Big(\cL_2 u_0\Big)\\
&\leq -2\gamma w -\nabla f\cdot \nabla w +P\Big( \|\varphi\|_{C^{m+3}(B(0,R))}, \|f\|_{C^{m+4}(B(0,R))},\|\Sigma\|_{C^{m+1}B(0,R)}\Big) e^{-2\gamma'' t}\,.
\end{align*}
From this we get
\begin{align*}
\sup_{x\in B(0,R)}w(x,t)\leq P\Big( \|\varphi\|_{C^{m+3}(B(0,R))},
  \|f\|_{C^{m+4}(B(0,R))},\|\Sigma\|_{C^{m+1}(B(0,R))}\Big) e^{-2\gamma'' t}\,.
\end{align*}
This shows that \eqref{e:underi} is true for $n=1$, $k=m+1$. Hence \eqref{e:underi} holds for all derivatives of $u_1$.

The analysis of $n\ge 2$ is similar to the case $n=1$ and
can be performed using induction. This completes the proof.
\end{proof}

\section{Formal Derivation of the Equations of the Measures}\label{eq:formaleq}

In this section, we aim to derive the equations \eqref{eq:nu0}-\eqref{eq:nu1} in a formal way. Observe that $\nu_0$ is a probability measure so the equation of $\nu_0$ can be derived from the empirical measure
$\frac{1}{N}\sum_i \delta(x-X_i(t))$ where each $X_i$ satisfies the transport equation \eqref{e:trans}. However, this cannot be generalized to the equation of $\nu_1$. Hence we adopt another different formal derivation as follows.

According to $u(x, t)=\mathbb{E}_x\varphi(X(t))$, we expect $u_0$ to be written as
\[
u_0(x, t)=\int_{\mathbb{R}^d} \varphi(y) G_0(dy, t; x).
\] 
According to the definition of $\nu_0$, one has
\[
\nu_0(\cdot, t)=\int_{\mathbb{R}^d} G_0(\cdot, t; x)\mu_0(dx),
\]
and thus  $G_0(dy, t; x)$ means the Green's function for the evolution of $\nu_0$ with initial condition $X(0)=x$, or $\delta(\cdot-x)$.  By the equation of $u_0(x, t)$, it is easy to find that $G_0$ satisfies
\begin{gather}\label{eq:evoGreen}
\partial_t G_0(\cdot, t; x)+\nabla f(x)\cdot \nabla_x G_0(\cdot, t; x)=0.
\end{gather}
Due to the Markovian property of the dynamics, we can easily infer that the measure $\nu_0$ satisfies
\begin{gather}
\nu_0(\cdot, t)=\int_{\mathbb{R}^d} G_0(\cdot, t-s; y) \nu_0(dy; s)=:\mathcal{I}_{t-s}^{(0)}\nu_0(\cdot; s).
\end{gather}
Here, $\mathcal{I}_{t-s}^{(0)}$ is the evolution operator. Using \eqref{eq:evoGreen}, one finds
\[
\partial_t \int_{\mathbb{R}^d} G_0(\cdot, t-s; x)\nu_0(dx; s)+\int_{\mathbb{R}^d} \nu_0(dx; s) \nabla f(x)\cdot \nabla_x \mathcal{I}_{t-s}^{(0)}\delta(\cdot-x)=0
\]
or
\[
\partial_t \nu_0(\cdot, t)-\int_{\mathbb{R}^d} \nabla_x\cdot\Big(\nabla f(x)\nu_0(dx; s)\Big) \mathcal{I}_{t-s}^{(0)}\delta(\cdot-x)=0.
\]
Since $ \mathcal{I}_{t-s}^{(0)}$ is independent of $x$, the second term is then reduced to
\[
-\int_{\mathbb{R}^d} \nabla_x\cdot\Big(\nabla f(x)\nu_0(dx; s)\Big) \mathcal{I}_{t-s}^{(0)}\delta(\cdot-x)
=-\mathcal{I}_{t-s}^{(0)}\int_{\mathbb{R}^d} \nabla_x\cdot\Big(\nabla f(x)\nu_0(dx; s)\Big) \delta(\cdot-x)
=-\mathcal{I}_{t-s}^{(0)}\nabla\cdot(\nabla f \nu_0(\cdot; s)).
\]
Taking $t\to s$, one obtains the equation for $\nu_0$.

Similarly, let $G_1(\cdot, t; x)$ satisfy the following inhomogeneous equation
\begin{gather}\label{eq:G1}
\partial_t G_1(\cdot, t; x)+\nabla f(x)\cdot \nabla_x G_1(\cdot, t; x)=-\frac{1}{4}\nabla \|\nabla f\|^2\cdot\nabla_x G_0(\cdot, t; x)
+\frac{1}{2}\mathrm{Tr}(\Sigma \nabla_x^2 G_0(\cdot, t; x)),~~~G_1(\cdot, 0; x)=0.
\end{gather}
Then, we have
\begin{gather}
u_1(x, t)=\int_{\mathbb{R}^d} \varphi(y) G_1(y, t; x)dy,
\end{gather}
and
\begin{gather}\label{eq:zeroG1}
\nu_1(\cdot, t)=\int_{\mathbb{R}^d} G_1(\cdot, t; x)\mu_0(dx).
\end{gather}

By the linearity, one has
\begin{gather}\label{eq:G1superposition}
\nu_1(\cdot, t)=\int_{\mathbb{R}^d} G_1(\cdot, t-s; x)\nu_0(dx, s)+\mathcal{I}_{t-s}^{(0)}\nu_1(\cdot, s).
\end{gather}
The first term arises from \eqref{eq:zeroG1} with zero initial data while the second term is from the homogeneous part with initial data $\nu_1(\cdot, s)$.
Setting $t\to t-s$ in \eqref{eq:G1}, multiplying $\nu_0(dx, s)$ and integrating, one has
\[
\partial_t\int_{\mathbb{R}^d} G_1(\cdot, t-s; x)\nu_0(dx, s)-\int_{\mathbb{R}^d} G_1(\cdot, t-s; x)\nabla\cdot(\nabla f(x)\nu_0(dx, s))=\mathcal{I}_{t-s}^{(0)}\left(\frac{1}{4}\nabla\cdot(\nabla\|\nabla f\|^2 \nu_0)+\frac{1}{2}\partial_{ij}(\Sigma_{ij}\nu_0)\right).
\]
Clearly, the second term $\mathcal{I}_{t-s}^{(0)}\nu_1(\cdot, s)$ satisfies
\[
\partial_t\mathcal{I}_{t-s}^{(0)}\nu_1(\cdot, s)-\nabla\cdot(\nabla f \mathcal{I}_{t-s}^{(0)}\nu_1(\cdot, s))=0.
\]
Adding the above two equations up and taking $t\to s$ yields
\[
\partial_t\nu_1-\nabla\cdot(\nabla f \nu_1)=\frac{1}{4}\nabla\cdot(\nabla\|\nabla f\|^2 \nu_0)+\frac{1}{2}\partial_{ij}(\Sigma_{ij}\nu_0).
\]

\begin{remark}
The generalization to $\nu_n$ for $n\ge 2$ is more involved and the equation for $\nu_n$ is similar to $\nu_1$. The key relation is some anology of \eqref{eq:G1superposition}, given by $\nu_n(\cdot, t)=\sum_{m=0}^n \int G_m(\cdot, t-s; x)\nu_{n-m}(dx,s)$ due to linearity. (In fact, one may also expand the Fokker-Planck equation for the diffusion approximation in terms of $\eta$ to obatin the equations for $\nu_n$. However, this type of derivation does not give the inisight into the dynamics.)
\end{remark}

\section{Computations for the Numerical Examples}\label{sec:examples-full}
In this appendix we include detailed computations used in the numerical examples in Section~\ref{sec:numer-exper}, where the domain is assumed to be one-dimensional ($d=1$). Note that $u_{0}$ is determined by the initial value problem
\begin{equation}
\label{eq:ivp-u0}
  \begin{aligned}
    &\partial_t u_0+f'\left( x \right)\partial_x u_0 \left( x,t \right)=0,\\
    & u_0(x,0)=\varphi(x).
  \end{aligned}
\end{equation}
The equation of the characteristic lines is
\begin{equation}
\label{eq:charac-line-eq-prim}
  \frac{\mathrm{d}x}{\mathrm{d}t}=f' \left( x \left( t \right) \right)
\end{equation}
which gives
\begin{equation}
\label{eq:charac-line-eq}
  t=\int_0^t \frac{\mathrm{d}x \left( t \right)}{f' \left( x \left( t \right) \right)}=\int_{x_0 \left( x,t \right)}^x \frac{\mathrm{d}\xi}{f' \left( \xi \right)}
\end{equation}
where $x_0=x_0 \left( x,t \right)$ is the intercept of the characteristic line passing through the point $\left( x,t \right)\in\mathbb{R}\times\mathbb{R}_{\geq 0}$. Therefore,
\begin{equation}
\label{eq:u0}
  \begin{aligned}
    u_0 \left( x,t \right)
    =\varphi \left( x_0 \left( x,t \right) \right).
  \end{aligned}
\end{equation}
Using implicit differentiation rules, one easily deduce from \eqref{eq:charac-line-eq} that
\begin{equation*}
    \partial_tx_0 \left( x,t \right)=-f' \left( x_0 \left( x,t \right) \right),\qquad
    \partial_xx_0 \left( x,t \right)=\frac{f' \left( x_0 \left( x,t \right) \right)}{f' \left( x \right)}
\end{equation*}
with which one easily verifies that \eqref{eq:u0} is the solution of the initial value problem \eqref{eq:ivp-u0}.

Furthermore, $u_1$ is determined by the initial value problem
\begin{equation}
\label{eq:ivp-u1}
  \begin{aligned}
    &\partial_t u_1+f'\left( x \right)\partial_x u_1 \left( x,t \right)=\cL_2u_0 \left( x,t \right),\\
    & u_1(x,0)=0.
  \end{aligned}
\end{equation}
Without loss of generality, we will assume $\Sigma=\frac{1}{4}$, which is the variance of a Bernoulli random variable with parameter $p=1/2$. Using \eqref{eq:u0} and \eqref{eq:differential-ops}, we have
\begin{equation*}
  \begin{aligned}
    \cL_2u_0 \left( x,t \right)&=-\frac{1}{2}f' \left( x \right)f'' \left( x \right)\partial_xu_0 \left( x,t \right)+\frac{1}{8}\partial_x^2u_0 \left( x,t \right)\\
    &=-\frac{1}{2}f' \left( x_0 \left( x,t \right) \right)\varphi' \left( x_0 \left( x,t \right) \right)f'' \left( x \right)+\frac{1}{8}\frac{\partial}{\partial x}\left[ \frac{f' \left( x_0 \left( x,t \right) \right)}{f' \left( x \right)} \right]\varphi' \left( x_0 \left( x,t \right) \right)+\frac{1}{8}\left[ \frac{f' \left( x_0 \left( x,t \right) \right)}{f' \left( x \right)} \right]^2\varphi''\left( x_0 \left( x,t \right) \right)
  \end{aligned}
\end{equation*}
in which the middle term in the right hand side can be further expanded into
\begin{equation*}
  \frac{1}{8}\varphi' \left( x_0 \left( x,t \right) \right)\frac{f' \left( x \right)f'' \left( x_0 \left( x,t \right) \right)\partial_xx_0 \left( x,t \right)-f'\left( x_0 \left( x,t \right) \right)f'' \left( x \right)}{\left[ f' \left( x \right) \right]^2}=\frac{1}{8}\varphi' \left( x_0 \left( x,t \right) \right)\frac{f' \left( x_0 \left( x,t \right) \right)\left[ f'' \left( x_0 \left( x,t \right) \right)-f'' \left( x \right) \right]}{\left[ f' \left( x \right) \right]^2}.
\end{equation*}
The equation of characteristic lines for \eqref{eq:ivp-u1} is the same as \eqref{eq:charac-line-eq}. Using the boundary condition $u_1 \left( x,0 \right)=0$, we have
\begin{equation*}
  \begin{aligned}
    &u_1 \left( x \left( t \right),t \right)=\int_0^t\cL_2u_0 \left( x \left( t \right),t \right)\,\mathrm{d}t\\
    &=-\frac{1}{2}f' \left( x_0 \right)\varphi' \left( x_0 \right)\int_0^tf'' \left( x \left( t \right) \right)\,\mathrm{d}t+\frac{1}{8}f' \left( x_0 \right)\varphi' \left( x_0 \right)\int_0^t \frac{f'' \left( x_0 \right)-f'' \left( x \left( t \right) \right)}{\left[ f' \left( x \left( t \right) \right) \right]^2}\,\mathrm{d}t+\frac{1}{8}\left[ f' \left( x_0 \right) \right]^2\varphi'' \left( x_0 \right)\int_0^t \frac{\mathrm{d}t}{\left[ f' \left( x \left( t \right) \right) \right]^2}\\
    &=:(\textrm{\RNum{1}})+(\textrm{\RNum{2}})+(\textrm{\RNum{3}})
  \end{aligned}
\end{equation*}
where we adopted the simplifying notation $x_0\equiv x_0 \left( x \left( t \right),t \right)$ for the constant along the characteristic line $x=x \left( t \right)$. By means of \eqref{eq:charac-line-eq-prim}, we can further simplify the three terms on the right hand side:
\begin{align*}
  (\textrm{\RNum{1}})&=-\frac{1}{2}f' \left( x_0 \right)\varphi' \left( x_0 \right)\int_{x_0}^x \frac{\mathrm{d}}{\mathrm{d}\xi} \log f' \left( \xi \right)\,\mathrm{d}\xi=-\frac{1}{2}f' \left( x_0 \right)\varphi' \left( x_0 \right)\log \frac{f' \left( x \right)}{f' \left( x_0 \right)},\\
(\textrm{\RNum{2}})&=\frac{1}{8}f' \left( x_0 \right)\varphi' \left( x_0 \right)\int_{x_0}^x \frac{f'' \left( x_0 \right)-f'' \left( \xi \right)}{\left[ f' \left( \xi \right) \right]^3}\,\mathrm{d}\xi,\\
(\textrm{\RNum{3}})&=\frac{1}{8}\left[ f' \left( x_0 \right) \right]^2\varphi'' \left( x_0 \right)\int_{x_0}^x \frac{\mathrm{d}\xi}{\left[ f' \left( \xi \right) \right]^3}.
\end{align*}
Therefore,
\begin{equation}
  \begin{aligned}
    u_1 \left( x,t \right)=-\frac{1}{2}f' \left( x_0 \left( x,t \right) \right)\varphi' \left( x_0\left( x,t \right) \right)\log \frac{f' \left( x \right)}{f' \left( x_0\left( x,t \right) \right)}&+\frac{1}{8}f' \left( x_0 \left( x,t \right)\right)\varphi' \left( x_0 \left( x,t \right)\right)\int_{x_0 \left( x,t \right)}^x \frac{f'' \left( x_0 \left( x,t \right) \right)-f'' \left( \xi \right)}{\left[ f' \left( \xi \right) \right]^3}\,\mathrm{d}\xi\\
   &+\frac{1}{8}\left[ f' \left( x_0 \left( x,t \right) \right) \right]^2\varphi'' \left( x_0 \left( x,t \right)\right)\int_{x_0 \left( x,t \right)}^x \frac{\mathrm{d}\xi}{\left[ f' \left( \xi \right) \right]^3}.
  \end{aligned}
\end{equation}
Alternatively, we can also write $u_1$ in the following equivalent form:
\begin{equation}
\label{eq:u1-alt}
  \begin{aligned}
    u_1 &\left( x,t \right)=-\frac{1}{2}f' \left( x_0 \left( x,t \right) \right)\varphi' \left( x_0\left( x,t \right) \right)\log \frac{f' \left( x \right)}{f' \left( x_0\left( x,t \right) \right)}+\frac{1}{8}f' \left( x_0 \left( x,t \right)\right)f'' \left( x_0 \left( x,t \right) \right)\varphi' \left( x_0 \left( x,t \right)\right)\int_{x_0 \left( x,t \right)}^x \frac{\mathrm{d}\xi}{\left[ f' \left( \xi \right) \right]^3}\\
   &-\frac{1}{8}f' \left( x_0 \left( x,t \right)\right)\varphi' \left( x_0 \left( x,t \right)\right)\int_{x_0 \left( x,t \right)}^x \frac{f'' \left( \xi \right)}{\left[ f' \left( \xi \right) \right]^3}\mathrm{d}\xi+\frac{1}{8}\left[ f' \left( x_0 \left( x,t \right) \right) \right]^2\varphi'' \left( x_0 \left( x,t \right)\right)\int_{x_0 \left( x,t \right)}^x \frac{\mathrm{d}\xi}{\left[ f' \left( \xi \right) \right]^3}\\
  =&-\frac{1}{2}f' \left( x_0 \left( x,t \right) \right)\varphi' \left( x_0\left( x,t \right) \right)\log \frac{f' \left( x \right)}{f' \left( x_0\left( x,t \right) \right)}+\frac{1}{8}f' \left( x_0 \left( x,t \right)\right)f'' \left( x_0 \left( x,t \right) \right)\varphi' \left( x_0 \left( x,t \right)\right)\int_{x_0 \left( x,t \right)}^x \frac{\mathrm{d}\xi}{\left[ f' \left( \xi \right) \right]^3}\\
   &-\frac{1}{16}f' \left( x_0 \left( x,t \right)\right)\varphi' \left( x_0 \left( x,t \right)\right)\left\{ \frac{1}{\left[ f' \left( x_0 \left( x,t \right) \right) \right]^2}-\frac{1}{\left[ f' \left( x \right) \right]^2}\right\}+\frac{1}{8}\left[ f' \left( x_0 \left( x,t \right) \right) \right]^2\varphi'' \left( x_0 \left( x,t \right)\right)\int_{x_0 \left( x,t \right)}^x \frac{\mathrm{d}\xi}{\left[ f' \left( \xi \right) \right]^3}.
  \end{aligned}
\end{equation}

\bibliographystyle{unsrt}
\bibliography{ref}

\end{document}